\documentclass{article}
\usepackage[margin=1in]{geometry}
\usepackage[utf8]{inputenc}
\usepackage[T1]{fontenc}
\usepackage{microtype}
\usepackage{graphicx}
\usepackage{adjustbox}
\usepackage{subfigure}
\usepackage{enumitem}
\usepackage{booktabs} %
\usepackage{makecell}
\usepackage{fullpage}
\usepackage{hyperref}
\usepackage{natbib}
\usepackage{authblk}
\usepackage{xcolor}
\usepackage{bbm}
\usepackage{algorithm}
\usepackage{algpseudocode}
\usepackage{xspace}
\usepackage{booktabs} %
\usepackage{multirow} %
\usepackage{subcaption}

\definecolor{Darkblue}{rgb}{0,0,0.4}
\hypersetup{colorlinks=true,
            citebordercolor={.6 .6 .6},
            linkbordercolor={.6 .6 .6},
            citecolor=Darkblue,
            urlcolor=blue,
            linkcolor=black}

\usepackage{amsmath}
\usepackage{amssymb}
\usepackage{mathtools}
\usepackage{amsthm}
\usepackage{thm-restate}  %

\usepackage[capitalize, noabbrev]{cleveref}

\theoremstyle{plain}
\newtheorem{theorem}{Theorem}[section]
\newtheorem{proposition}[theorem]{Proposition}
\newtheorem{lemma}[theorem]{Lemma}

\theoremstyle{definition}
\newtheorem{definition}[theorem]{Definition}

\theoremstyle{remark}

\usepackage[textsize=tiny]{todonotes}

\usepackage{nicefrac}
\usepackage{dsfont}
\usepackage{comment}

\usepackage{typed-checklist}

\newcommand{\Secref}[1]{\hyperref[#1]{Section \ref*{#1}}}
\newcommand{\Appref}[1]{\hyperref[#1]{Appendix \ref*{#1}}}

\newcommand{\squishlist}{
 \begin{list}{$\bullet$}
  { \setlength{\itemsep}{0pt}
     \setlength{\parsep}{3pt}
     \setlength{\topsep}{3pt}
     \setlength{\partopsep}{0pt}
     \setlength{\leftmargin}{1.5em}
     \setlength{\labelwidth}{1em}
     \setlength{\labelsep}{0.5em} } }

\newcommand{\squishlisttwo}{
 \begin{list}{$\bullet$}
  { \setlength{\itemsep}{0pt}
    \setlength{\parsep}{0pt}
    \setlength{	opsep}{0pt}
    \setlength{\partopsep}{0pt}
    \setlength{\leftmargin}{2em}
    \setlength{\labelwidth}{1.5em}
    \setlength{\labelsep}{0.5em} } }

\newcommand{\squishend}{
  \end{list}  }

\newcommand{\calO}{\mathcal{O}}

\DeclareMathOperator*{\E}{\mathbb{E}}

\DeclareMathOperator*{\argmax}{\arg\!\max}

\usepackage[suppress]{color-edits}
\addauthor{mb}{green!60!black}
\addauthor{na}{purple!60!black}
\addauthor{nh}{blue}

\newcommand{\bO}[1]{\calO\left(#1\right)}
\newcommand{\bOtilde}[1]{\tilde{\calO}\left(#1\right)}

\newcommand{\acts}{\mathcal{A}}
\newcommand{\abs}[1]{\left|#1\right|}

\newcommand{\bP}{\mathbb{P}}

\begin{document}

\title{Provably Optimal Learning Algorithms for Assistance Games}
\date{}
\author[]{Nivasini Ananthakrishnan}
\author[]{Mark Bedaywi}
\author[]{Michael I. Jordan}
\author[]{Stuart Russell}
\author[]{Nika Haghtalab}
\affil[]{University of California, Berkeley\\ \texttt{\{nivasini,mark\_bedaywi,michael\_jordan,russell,nika\}@berkeley.edu}}
\maketitle
\begin{abstract}
This paper studies an online variant of the \emph{assistance games} framework, where an informed agent and an uninformed agent repeatedly interact over $T$ timesteps to optimize a common reward function. While the informed agent (the human) observes a latent state of the world, the uninformed agent (the assistant) observes only the human's actions. We provide the first provably efficient learning algorithms for repeated assistance games. We introduce the notion of \emph{assistance regret}: the gap between the cumulative utility of interactions and that of the optimal joint policies in hindsight, which map latent states to action pairs. We present decentralized algorithms for both the human and the assistant that achieve a $(1 - 1/e)$-approximate assistance regret rate of $\bOtilde{T^{3/4}}$, with runtime polynomial in the size of the action and state spaces. These algorithms are general; in particular, they accommodate any no-regret algorithm for the assistant.
We prove that achieving a regret approximation factor better than $(1 - 1/e)$ is computationally intractable. Furthermore, we demonstrate how these generic no-regret algorithms can be tailored to a pseudo-decentralized setting---using a shared random string---to achieve a rate of $\bOtilde{T^{1/2}}$, optimal up to logarithmic factors.

\end{abstract}

\newcommand{\alg}{\mathrm{Alg}}
{
\section{Introduction}

Consider a repeated interaction between two cooperative agents who share a common objective but have asymmetric access to information. One agent observes a changing latent state—such as a preference, type, or objective—while the other must act based only on indirect signals produced by the first agent.
Such settings arise naturally in abstractions of human--assistant interaction, cooperative multi-agent systems, training of assistive AI, and emergent communication, where the former agent represents a human principal with private preferences and the latter represents an AI assistant seeking to act on the human’s behalf.

Despite the absence of strategic misalignment in utilities, coordination in these environments poses a fundamental challenge.
On the one hand, the lack of a shared frame of reference or common language means that the agents must learn how to communicate by observing each other’s actions over time.
On the other hand, actions simultaneously generate utility, causing them to serve a dual role as  instruments for achieving reward and signals for conveying the latent state.
This tension between informativeness and utility lies at the core of the problem.

One theoretical perspective for capturing these challenges is the \emph{assistance games} framework, also known as \emph{Cooperative Inverse Reinforcement Learning}~\citep{hadfield2016cooperative}.
This framework models interactions between a human and an assistive AI system as a cooperative game with partial observability.
Prior work has studied equilibria and structural properties of optimal assistive policies—for example, tradeoffs between the informativeness and utility of actions—but the algorithmic problem of computing and learning such policies has been  unexplored.

In this paper, we introduce an \textbf{online variant of assistance games and give computationally efficient, decentralized algorithms for learning near-optimal policies for both the human and the assistant.}

Our model captures repeated interaction between an informed agent (the human) and an uninformed agent (the assistant).
In each round $t$, a latent state $\theta^{(t)} \in \Theta$—possibly drawn from a nonstationary process—is realized and observed only by the human.
This state represents private information about the human’s preferences or objectives.
The human selects an action $a_H^{(t)}$ from a finite action space $\mathcal{A}_H$ of size $M_H$.
The assistant observes $a_H^{(t)}$, but not $\theta^{(t)}$, and responds with an action $a_A^{(t)}$ from a finite action space $\mathcal{A}_A$ of size $M_A$.
Both agents then receive a common reward
$r\!\left(a_H^{(t)}, a_A^{(t)}; \theta^{(t)}\right),
$
which depends on the joint actions and the human’s private state.

To measure performance, we consider the space of \emph{joint} human--assistant policies: pairs $(\pi_H, \pi_A)$ where a human policy $\pi_H : \Theta \to \mathcal{A}_H$ maps latent states to human actions and an assistant policy $\pi_A : \mathcal{A}_H \to \mathcal{A}_A$ maps human actions to assistant actions. We measure success via \emph{$\alpha$-approximate assistance regret} (\Cref{def:regret}): the gap between the agents' cumulative reward and an $\alpha$ factor of the cumulative reward of the best joint policy in hindsight.

Our main result gives a pair of decentralized learning algorithms for the human and the assistant that achieve $(1-1/e)$-approximate assistance regret at a rate of $\bOtilde{T^{3/4}}$ in time $\mathrm{poly}(N, M_H, M_A)$ (\Cref{prop:natural-regret}); with light initial coordination---a shared encoding from human action sequences to assistant policies---this rate improves to $\bOtilde{\sqrt T}$, optimal in $T$ up to logarithmic factors (\Cref{thm:main-algorithm}). Two technical ingredients drive the result. First is a reduction casting joint policy optimization in assistance games as online submodular maximization with matroid constraints (\Cref{lem:assistance-partition-matroid,prop:assistance-matroid-reduction}). This reduction enables computational tractability despite exponential size of the joint policy space by leveraging the structure of submodular functions. Second is a regret-decomposition lemma (\Cref{lem:stable-adaptive}) that bounds assistance regret by the centralized algorithm's external regret, its number of policy switches, and the assistant's tracking regret against a moving target. This decomposition guides our design toward two structural properties: \emph{stability} (few policy switches) for the human and \emph{adaptivity} (low tracking regret) for the assistant. The approximation factor of $1-1/e$ is tight: no efficient algorithm achieves sublinear $\alpha$-approximate assistance regret for $\alpha > 1-1/e$ unless $\mathsf{RP} = \mathsf{NP}$ (\Cref{thm:no-efficient}).

}

\section{Related Work}

The idea of modeling the problem of AI assistance mathematically had been first proposed almost ten years ago in work including \citet{fern2014decision} and \citet{hadfield2016cooperative}. \citet{shah2020benefits} show that such an assistant would be able to better satisfy human preferences than traditional reward learning algorithms, and \citet{hadfield2016cooperative, hadfield2017off} show that an optimal assistant would avoid issues of misalignment to human values. \citet{malik2018efficient} and \citet{laidlaw2025assistancezero} propose algorithms for solving assistance games that are empirically performant, but the problem of designing learning algorithms that provably play such games optimally remains open. Similar to us, previous work has also taken the approach of reducing the assistance games problem to other classes of problems --- particularly to classes of Partially Observable Markov Decision Processes (POMDPs). We introduce a new reduction with a mathematical structure that allows us to demonstrate provable optimality guarantees. A particular subclass of assistance games are \emph{communication games}. We describe related works on this literature in more detail in \Cref{sec:add_related}.

Our work is related to literature on equilibrium computation. We can frame our goals is decentralized learning of optimal equilibria in assistance games. Computing the optimal equilibrium is computationally intractable~\citep{gilboa1989nash} in general, including in games of common interest~\citep{chu2001np, conitzer2006computing}. In common interest games, the optimal equilibrium is also each player's Stackelberg equilibrium. The computational tractability of computing the Stackelberg equilibrium is shown to depend on the geometry of the game and can be intractable in general~\citep{letchford2009learning, peng2019learning}. For games where the Stackelberg equilibrium can be computed efficiently, convergence to the Stackelberg equilibrium can be achieved through dynamics where one player's learning dynamic is more stable than the other~\citep{brown2024learning, zrnic2021leads}. The learning dynamics we propose for communication games also satisfy this property. 

Other areas studying cooperative interactions between agents include work on team decision theory (e.g.~\cite{radner1962team, ho1972team, nayyar2013decentralized, mahajan2016decentralized,malikopoulos2022team}) which takes a control theoretic approach to studying optimal cooperation, and work on human-AI collaboration studying when human-AI systems achieve complementarity i.e., better performance than the sum of individual components (e.g.,~\cite{green2019principles, bansal2021does, wilder2020learning, steyvers2022bayesian,donahue2022human,athey2020allocation, alur2024human, greenwood2025designing,collina2026collaborative})

\section{Model and Preliminaries}\label{sec:model}

The online assistance game is defined by a preference space $\Theta$ of size $N$, two sets of actions $\mathcal{A}_H$ and $\mathcal{A}_A$ of sizes $M_H$ and $M_A$ respectively, and a bounded reward function $r: \mathcal{A}_H \times \mathcal{A}_A \times \Theta \to [0, 1]$. Two agents, an assistant and a human, repeatedly play the game over $T$ rounds. First, nature fixes preferences $\theta^{(1)}, \dots, \theta^{(T)} \in \Theta$ before play. Then, every round $t \in [T]$ of the assistance game involves the following steps:

\begin{enumerate}[itemsep=0pt, topsep=0pt, leftmargin=*]    
    \item The human observes $\theta^{(t)}$ and plays an action $a_H^{(t)} \in \mathcal{A}_H$.
    
    \item The assistant sees only the action $a_H^{(t)}$ the human takes, and not the preference $\theta^{(t)}$. The assistant then takes the action $a_A^{(t)} \in \mathcal{A}_A$ in response.

    \item Both the assistant and the human receive the reward $r\left(a_H^{(t)}, a_A^{(t)}; \theta^{(t)}\right)$.
\end{enumerate}

The parameter $\theta^{(t)}$ captures the human's preference at round $t$. We allow the preference $\theta^{(t)}$ to change over time, allowing robustness to changing human preferences. However, we impose the restriction that the sequence of preferences is fixed before the game begins. That is, the preference at each round cannot be selected adaptively based on previous rounds of interaction. We call this the \emph{oblivious} setting. The stronger model of an adversary who can select the preference $\theta^{(t)}$ adversarially based on previous rounds of interaction is called the \emph{adaptive} setting. We prove in \Cref{sec:adaptive-adversary} that there are fundamental limitations to learning optimal assistance in the adaptive setting.

The assistant receives bandit feedback: it observes the action pair $a_H^{(t)}, a_A^{(t)}$ and the corresponding reward $r\left(a_H^{(t)}, a_A^{(t)}; \theta^{(t)}\right)$, but because it does not have access to $\theta^{(t)}$, it cannot reconstruct counterfactual rewards. In contrast, the human, because they also observe $\theta^{(t)}$, receives full-information feedback.

In the normal form representation of the assistance game, the human's policy space is the set $\Pi_H$ of mappings from the human's private preference to human's actions, $\pi_H: \Theta \to \mathcal{A}_H$, and the assistant's policy space is the set $\Pi_A$ of mappings from an observed human action to an assistant action, $\pi_A: \mathcal{A}_H\to \mathcal{A}_A$. For notational convenience, we write $r(\pi_H, \pi_A; \theta)$ to mean the reward the human and assistant receive when they play policy $\pi_H$ and $\pi_A$ when the human's preferences are $\theta$, mathematically $r(\pi_H(\theta),~\pi_A(\pi_H(\theta));~\theta)$.

We introduce a notion of (approximate)regret to measure the joint success of the human and assistant in the online assistance game. We call this the \emph{assistance regret} and simply refer to this as regret in the remainder of the paper. Assistance regret is defined in the following definition.

\begin{definition}[$\alpha$-Assistance regret]\label{def:regret}
    For a fixed $\alpha \in [0, 1]$, the $\alpha$-assistance regret of a sequence $\chi = \left (\theta^{(t)}, \pi_H^{(t)}, \pi_A^{(t)}\right )_{t=1}^T$, written $R^{\alpha}_T(\chi)$, is:
    \begin{align*}
    \alpha \left(\max_{\pi_H^* \in \Pi_H, \pi_A^* \in \Pi_A}\sum_{t=1}^T r\left(\pi_H^*, \pi_A^*; \theta^{(t)}\right)\right)
    -  \sum_{t=1}^T r\left(\pi_H^{(t)}, \pi_A^{(t)}; \theta^{(t)}\right).
    \end{align*}
   We suppress $\theta^{(t)}$ in the regret notion when it is clear from context and denote $r_t(\pi_H, \pi_A) = r(\pi_H, \pi_A; \theta^{(t)})$. We also refer to $R_T(\mathrm{Alg})$ as the min-max $\alpha$-assistance regret of algorithm $\mathrm{Alg}$.
\end{definition}

\subsection{Submodular Maximization With Matroid Constraints}

Our results will rely on tools from submodular maximization and matroid theory. We will describe the necessary tools in this section.
Given an arbitrary set of elements $\mathcal{U}$, a set function over $\mathcal{U}$ is a mapping from subsets in $\mathcal{U}$ to real numbers, $f: 2^\mathcal{U} \to \mathbb{R}$.

A set function is said to be monotone when adding any element can only increase the value of the function. That is $f$ is monotone when for all $S \subseteq \mathcal{U}$ and all $x \notin S$, it holds that $f(S \cup \{x\}) \geq f(S)$.

A set function is said to be submodular when the marginal improvement of adding an element can only decrease
as the size of the set increases.
\begin{definition}[Submodularity]
    \label{def:submodularity}
    Let $f$ be a set function on a set $\mathcal{U}$. $f$ is said to be \textit{submodular} when,
    for all $S' \subseteq S \subseteq \mathcal{U}$ and all $x \notin S$, the increase in value
    after adding $x$ to $S$ is smaller than the increase in value after adding $x$ to $S'$.
    That is, $f(S \cup \{x\}) -f(S) \leq f(S' \cup\{x\}) - f(S')$.
\end{definition}

We will consider the constrained maximization of submodular functions under matroid constraints.

\begin{definition}[Matroids]
    A pair $\mathcal{M} = (\mathcal{U}, \mathcal{I})$ is a \textit{matroid} when $\mathcal{U}$ is a finite set, and $\mathcal{I} \subseteq 2^\mathcal{U}$ is a collection of
    subsets of $\mathcal{U}$ such that 1) For all $I \in \mathcal{I}$, if $J \subseteq I$, then $J \in \mathcal{I}$ and 2) For all $I, J \in \mathcal{I}$, if $\abs{J} < \abs{I}$, then there exists some $e \in I \setminus J$
    such that $J \cup \{e\} \in \mathcal{I}$.

    Any set in $I \in \mathcal{I}$ is said to be an \textit{independent set}. The maximum size of an independent set in $\mathcal{I}$ is said to be the \textit{rank} of the matroid $\mathcal{M}$.
\end{definition}

Our results actually rely on a special subclass of submodular functions i.e., \emph{weighted threshold potentials} (a sum of capped linear functions; \Cref{def:weighted-coverage}) and on a special subclass of matroids i.e., \emph{partition matroids} (\Cref{def:partition-matroids}). We provide more details on these subclasses in \Cref{app:prelim-extras}.

\subsection{Background on Online Learning}
Regret notions and algorithms originating from the online learning literature will play a central role in our paper.

In a setting with action set $[K]$ and $T$ rounds, for sequences of actions and reward functions $(a_t, r_t)$ where $a_t \in [K]$ and $r_t: [K] \to [0,1]$, we can define the following notions of regret that we will make extensive use of in framing our results.

\begin{definition}[External regret and number of switches]\label{def:ext_regret}
    The $\alpha$-approximate external regret of a sequence $\chi = (a_t, r_t)_{t=1}^T$ is $R_T^{\alpha, \mathrm{ext}}(\chi)= \alpha \max_{a \in [K]} \sum_{t=1}^{T} r_t(a) - \sum_{t=1}^{T} r_t(a_t)$. The number of switches in the sequence is $\sum_{t=2}^{T} \mathbf{1}\{a_t \neq a_{t-1}\}$.
\end{definition}

\begin{definition}[Tracking regret with $p$ segments]\label{def:tracking_regret}
    Tracking regret~\citep{herbster1998tracking} measures a learner's ability to compete with a sequence of a $p$-times changing benchmark of actions rather than a single best fixed action. For a sequence  
$\chi = (a_t, r_t)_{t=1}^T$, the tracking regret is defined as :   
    \[
    R_T^{\text{track}}(\chi; p) = \max_{\substack{s_1 < \dots < s_{p+1} \in [T+1]  \\ s_1 = 1, s_{p+1}=T +1 \\ b_1, \dots, b_p \in [K]}}   
    \sum_{i=1}^{p} \sum_{t=s_i}^{s_{i+1} - 1} r_t(b_i) - \sum_{t=1}^{T} r_t(a_t). 
\]
    This regret formulation is useful in changing environments where the optimal action varies over time.
\end{definition}

\emph{Feedback models.} Online learning algorithms select action $a_t$ in a round $t$ based on the available history consisting of actions and feedback of previous rounds. In a \emph{full-information} setting the entire reward function $r_t$ is revealed as feedback at the end of round $t$. In a $\emph{bandit}$ setting, only the reward of the selected action $r_t(a_t)$ is revealed as feedback.

\section{Main Results and Techniques}
\label{sec:overview}

Our main results provide pairs of decentralized learning algorithms for the human and the assistant that are computationally efficient and result in sublinear $(1-1/e)$-approximate assistance regret. We first state our main theorems, discuss the optimality of the guarantees stated in these theorems, and finally describe the main ideas behind constructing these algorithms.

Our first theorem constructs decentralized algorithms that use quite general regret minimization algorithms and achieve a $\bOtilde{T^{3/4}}$ rate of $(1-1/e)$ -approximate assistance regret.

\begin{restatable}{theorem}{NaturalRegret}\label{prop:natural-regret}
    There are $\mathrm{poly}(M_H, M_A,N,T)$ time learning algorithms for the human and assistant that achieve no-$(1 - \nicefrac{1}{e})$-approximate regret at a rate of $\bO{M_H^{5/4} M_A^{3/4} T^{3/4} \log(M_A T) + M_A M_H \sqrt{T}}$.
    
    Moreover, any assistant algorithm with minmax optimal tracking regret rates can be used to obtain this bound.
\end{restatable}

Our second theorem constructs decentralized algorithms that are more tailored to the assistance game and achieve a tighter $\bOtilde{\sqrt{T}}$ rate of $(1-1/e)$-approximate assistance regret. These algorithms require some initial synchronization between the human and assistant in the form of holding a shared mapping from sequences of human actions to the policy space. 

\begin{restatable}{theorem}{MainLearningAlgorithm}
    \label{thm:main-algorithm}
    There are $\mathrm{poly}(M_H, M_A,N,T)$ time learning algorithms for the human and assistant that make use of a shared map $\phi: \mathcal{A}_H^* \rightarrow \Pi_A$ from sequences of human actions to assistant policies and achieve no-$(1 - \nicefrac{1}{e})$-approximate regret at a rate of $\bO{M_A M_H \sqrt{T} + M_H\sqrt{M_H M_A T}\, \log(M_A T)\log(T)}$.
\end{restatable}

We will describe the main ideas behind constructing these algorithms in this section, deferring the full proofs to \Cref{proof:natural-regret,proof:main-algorithm}. But first, let us discuss the optimality of the regret guarantees. Note that guarantees we achieve in both settings of with and without initial synchronization is for assistance regret with a level of $(1-1/e)$-approximation. It turns out that it is not tractable to achieve sublinear assistance regret with a better approximation factor, which we state in \Cref{thm:no-efficient}.

\begin{restatable}{theorem}{NoEfficient}
    \label{thm:no-efficient}
    Unless $\text{RP} = \text{NP}$, for any $\alpha > 1 - \nicefrac{1}{e}$, any  algorithm that runs in time  $\text{poly}(N, M_H, M_A)$ per iteration has $\alpha$-approximate assistance regret such that either $R^\alpha_T \in \omega(T^{1 - \epsilon})$ for all $\epsilon > 0$ or $R^\alpha_T \notin  \text{poly}(N,M_H, M_A)$.
\end{restatable}

We prove this in \Cref{sec:lower-bound}. Given this approximation factor, the regret bound in \Cref{thm:main-algorithm} has a $\sqrt{T}$ dependence on the number of rounds $T$ up to logarithmic factors, which is optimal due to standard minmax lower bounds for regret minimization. This is formalized for assistance games in \Cref{prop:t_lower_bound}.

\subsection{Central Challenges and Techniques}
There are two central challenges that the human-assistant system needs to overcome to have low assistance regret. The first is \emph{computational}: the space of human-assistant policy pairs is exponential in size. The second is a \emph{coordination} challenge: the human and assistant need to learn jointly optimal strategies despite learning being decentralized. We address these challenges in more detail in \Cref{sec:centralized} and \Cref{sec:decentralized} respectively.

\paragraph{Challenge 1: Computational tractability of joint optimization (\Cref{sec:centralized}).}
The space of all joint human-assistant policy pairs consists of all pairs of mappings from $\Theta$ to $\mathcal{A}_H$ and $\mathcal{A}_H$ to $\mathcal{A}_A$. This space has size $M_H^N \cdot M_A^{M_H}$, which is exponential in the game dimensions $N$ and $M_H$. To isolate this computational challenge from the coordination challenge, we first consider an idealized \emph{centralized} setting (\Cref{def:central_communication}) where a meta-player jointly chooses both the human's and the assistant's policies.
We provide a centralized algorithm with $\mathrm{poly}(M_H, M_A, N, T)$ runtime and prove that it has sublinear $(1-1/e)$-assistance regret ({\Cref{sec:centralized}}). The key insight is a reduction of joint policy optimization in assistance games to maximization of a special subclass of submodular functions---weighted threshold potentials (\Cref{def:weighted-coverage})---over a partition matroid (\Cref{def:partition-matroids}). This lets us build upon existing  algorithms on online convex optimization for submodular maximization~\cite{salem2024online}.\footnote{While any algorithm for online submodular optimization would be effective here, our choice to build upon \cite{salem2024online} is guided by additional properties we will need for the decentralized setting. In particular, because this algorithm is based on online convex optimization, we can make it such that the total number of steps where the selected policies switch is sublinear.} Even in this idealized centralized setting, the approximation factor of $(1-1/e)$ is tight without compromising on the computational efficiency.

\begin{figure*}[t]
\captionsetup{type=algorithm}
\caption{Generic algorithms in the centralized setting $(\alg_C)$ and decentralized setting ($\alg_H$ and $\alg_A$).}
\label{alg:CAH}
\centering
\setlength{\fboxsep}{2pt}
\setlength{\fboxrule}{0.4pt}

\fbox{%
\begin{minipage}[t][2.1cm][t]{0.29\textwidth}
\footnotesize
\textbf{$\alg_C$ (centralized)}\\[-0.4em]
\begin{algorithmic}[1]
\State Play $(\pi_H^{(t)},\pi_A^{(t)})$
\State Observe $\theta^{(t)}$
\State Receive reward $r_t(\pi_H^{(t)},\pi_A^{(t)})$
\end{algorithmic}
\end{minipage}}%
\hfill
\fbox{%
\begin{minipage}[t][2.1cm][t]{0.29\textwidth}
\footnotesize
\textbf{\(\alg_A\) (decentralized assistant)}\\[-0.4em]
\begin{algorithmic}[1]
\State Observe $\pi_H^{(t)}(\theta^{(t)})$
\State Play $\pi_A^{(t)}$
\State Receive and see $r_t(\pi_H^{(t)},\pi_A^{(t)})$
\end{algorithmic}
\end{minipage}}%
\hfill
\fbox{%
\begin{minipage}[t][2.1cm][t]{0.38\textwidth}
\footnotesize
\textbf{\(\alg_H\) (decentralized human)}\\[-0.4em]
\begin{algorithmic}[1]
\State Observe $\theta^{(t)}$
\State $(\overline{\pi}_H^{(t)},\overline{\pi}_A^{(t)}) \gets \alg_C(\theta^{(1)},\ldots,\theta^{(t-1)})$
\State Play $\overline{\pi}_H^{(t)}(\theta^{(t)})$
\end{algorithmic}
\end{minipage}}%
\end{figure*}

\paragraph{Challenge 2: Coordination through decentralized learning (\Cref{sec:decentralized}).}In the actual online assistance game, the human and assistant learn in a decentralized manner. That is, their strategies must recover the joint optimal without being jointly optimized. We call this the \emph{coordination} challenge. We show that given any online algorithm for the centralized setting ($\alg_C$), and any online algorithm for the assistant's problem ($\alg_A$), we can construct an online algorithm in the decentralized setting for the human that overcomes the coordination challenge. In particular, we  decompose the assistance regret of the decentralized pair of policies into the sum of the centralized algorithm's external regret and a coordination cost that depends on tracking regret of $\alg_A$ and the number of switches $\alg_C$ makes in its choices.  This is formalized in the following lemma.

\begin{restatable}{lemma}{InitialSetup}
    \label{lem:stable-adaptive}
    Given any online algorithm $\alg_C$ in the centralized setting, and any online algorithm $\alg_A$ for the the assistant, there is a (decentralized) human algorithm $\alg_H$, 
that for every $\delta > 0$, achieve assistance regret of at most
    \begin{equation*}
        R^{\alpha}_T \leq R_T^{\alpha,\textrm{ext}}(\alg_{C}) + R_T^{track}(\alg_A; S_T(\alg_C; \delta)) + \delta T,
    \end{equation*}
 where $S_T(\alg_C; \delta)$ is an upper bound on the number of times $\alg_C$ switches strategies that holds with probability at least $1 - \delta$ and $R^{\alpha, ext}_T, R^{track}_T$ denote $\alpha$-approximate-external and (exact)-tracking regrets, respectively. Furthermore, $\alg_H$, described 
in Algorithm~\ref{alg:CAH}, is a simple wrapper 
around the centralized algorithm $\alg_C$.
\end{restatable}

The decomposition (proved in \Cref{proof:stable-adaptive}) guides our algorithm design: we want a centralized algorithm $\alg_C$ with low external regret \emph{and} few switches (so $S_T(\alg_C; \delta)$ is small), and an assistant algorithm $\alg_A$ with low tracking regret. Both \emph{stability} (few switches) and \emph{adaptivity} (low tracking regret) are well-studied notions in the online learning literature. We exploit the fact that our centralized algorithm of \Cref{prop:centralized-regret} is based on online convex optimization to make it \emph{stable} (\Cref{sec:stable-centralized}, \Cref{prop:stable-centralized}), and we construct an efficient adaptive assistant algorithm using the machinery tracking-regret (\Cref{sec:adaptive-assistant}).

\section{Computationally Efficient Algorithms for Centralized Learning in  Assistance Games}\label{sec:centralized}
This section addresses the computational challenge identified in \Cref{sec:overview}. We isolate it from the coordination challenge by studying an idealized \emph{centralized} setting in which a meta-player jointly chooses both the human and assistant policies. We develop a centralized algorithm that achieves $(1-1/e)$-approximate assistance regret at a rate of $\bO{\sqrt{T}}$ in $\mathrm{poly}(M_H, M_A, N, T)$ time, despite the exponentially large joint policy space. This algorithm will serve as the building block for the decentralized algorithm in \Cref{sec:decentralized}.

In the centralized online assistance game (\Cref{def:central_communication}), a meta-player chooses a joint policy pair $(\pi_H^{(t)}, \pi_A^{(t)}) \in \Pi_H \times \Pi_A$ at each round before the realized preference $\theta^{(t)}$ is revealed, and receives reward $r(\pi_H^{(t)}, \pi_A^{(t)}; \theta^{(t)})$. The meta-player's external regret in this game---against the best fixed pair $(\pi_H^*, \pi_A^*)$ in hindsight---is exactly the assistance regret (\Cref{def:regret}), so any low-external-regret algorithm in the centralized game yields low assistance regret.

Our main result for the centralized setting is a tractable algorithm with sublinear $(1-1/e)$-approximate assistance regret.

\begin{restatable}{proposition}{CentralizedAssistanceRegret}
    \label{prop:centralized-regret}
    There is a $\mathrm{poly}(M_H, M_A,N,T)$ time algorithm for the meta-player in the centralized assistance game (\Cref{def:central_communication}), such that the meta-player's expected $(1-\nicefrac{1}{e})$-assistance regret is at most $\bO{M_H\sqrt{T \log M_A}}$.
\end{restatable}
  
The main proof idea is a reduction of the assistance game to submodular maximization under matroid constraints, which lets us apply algorithms from prior work. We next develop this reduction in \Cref{sec:reduction}.

\subsection{A Reduction to Online Submodular Maximization} \label{sec:reduction}

To derive the reduction, we first represent the policy space as a partition matroid. Next, we show that the optimization objective over this representation satisfies structural properties that enable tractable online submodular optimization via the algorithms of \citet{salem2024online}.

\paragraph{Policy space representation.} We first note that joint optimization over human-asssistant policy pairs$(\pi_H, \pi_A)$ reduces to optimization over assistant policies alone. Given any $\pi_A$, the meta-player can pair it with the best-response human policy $\pi_H^*(\theta) := \arg\max_{a_H} r(a_H, \pi_A(a_H); \theta)$. We represent each assistant policy $\pi_A : \mathcal{A}_H \to \mathcal{A}_A$ as the set of human-assistant action pairs $\{(a_H, \pi_A(a_H)) : a_H \in \mathcal{A}_H\}$ it induces, and the resulting policy space as a partition matroid.

\begin{definition}[Assistance Matroid]
    \label{def:assistance-matroid}
    Given an assistance game $G$ with human action space $\mathcal{A}_H$ and assistant action space $\mathcal{A}_A$, define the \emph{assistance matroid} $\mathcal{M}_G = (\mathcal{U}_G, \mathcal{I}_G)$ where:
    \begin{itemize}[noitemsep,topsep=2pt,leftmargin=1.5em]
        \item $\mathcal{U}_G = \mathcal{A}_H \times \mathcal{A}_A$.
        \item $\mathcal{I}_G = \left\{S \subseteq \mathcal{U}_G : \text{for all } \left(a_H^{(1)}, a_A^{(1)}\right), \left(a_H^{(2)}, a_A^{(2)}\right) \in S, \text{ we have } a_H^{(1)} \neq a_H^{(2)}\right\}$.
    \end{itemize}
\end{definition}

The assistance matroid is a partition matroid (\Cref{def:partition-matroids}) of rank $M_H$, with parts $\mathcal{U}_{a_H} = \{a_H\} \times \mathcal{A}_A$ and capacities $d_{a_H}=1$ (\Cref{lem:assistance-partition-matroid}, proved in \Cref{proof:assistance-partition-matroid}).

\paragraph{Optimization objective and its structural properties.}
For each preference $\theta$, the \emph{value of assistance function} $V_\theta$ maps an independent set $I$ (equivalently, an assistant policy) to the reward achieved by pairing it with the best-response human policy under $\theta$. This is the objective function we maximize over the matroid representation of the policy space.
\begin{definition}[Value of Assistance Function]
    \label{def:value-of-assistance}
    Given an assistance game $G$, the value of assistance function for a human preference $\theta \in \Theta$ is the function $V_\theta: \mathcal{I}_G \to \mathbb{R}$ defined over independent sets $\mathcal{I}_G$ of the assistance matroid $\mathcal{M}_G=(\mathcal{U}_G, \mathcal{I}_G)$, such that $V_\theta(\emptyset) = 0$ and for all non-empty $S \in \mathcal{I}_G$,
    \begin{equation*}
        V_{\theta}(S) = \max_{(a_H, a_A) \in S} r(a_H, a_A; \theta).
    \end{equation*}
\end{definition}

The following proposition formalizes the reduction: the centralized online assistance game reduces to online maximization of value of assistance functions over the assistance matroid. The proof is deferred to \Cref{proof:assistance-matroid-reduction}.

\begin{proposition}
    \label{prop:assistance-matroid-reduction}
    The centralized online assistance game $G$ reduces to an online optimization over the family of value of assistance functions $V_\theta$ defined over independent sets of the assistance matroid $\mathcal{M}_G$. {That is, the following conditions are met
    \begin{enumerate}[noitemsep,topsep=2pt,leftmargin=1.5em]
        \item For every independent set $I \in \mathcal{I}_G$ and preference $\theta \in \Theta$, there exists a human-assistant policy pair $\pi_H, \pi_A \in \Pi_H, \Pi_A$ achieving reward at least $V_\theta(I)$.
        \item For every assistant policy $\pi_A \in \Pi_A$, there exists an independent set $I \in \mathcal{I}_G$ such that for all preferences $\theta \in \Theta$ and all human policies $\pi_H \in \Pi_H$, the reward the policy pair $\pi_H, \pi_A \in \Pi_H, \Pi_A$ achieves is at most $V_\theta(I)$.
    \end{enumerate}}
\end{proposition}

The value of assistance function further satisfies structural properties that enable tractable online optimization via the algorithms of \citet{salem2024online}.

\begin{lemma}[Structural properties of $V_\theta$, informal]
\label{lem:assistance-properties-informal}
For every preference $\theta \in \Theta$, the value of assistance function $V_\theta$ is a weighted threshold potential (\Cref{def:weighted-coverage}), and its concave relaxation (\Cref{def:concave-relaxation}) is $1$-Lipschitz.
\end{lemma}
The formal statements (\Cref{lem:assistance-is-weighted-threshold,lem:assistance-is-Lipschitz}) and proofs are deferred to \Cref{proof:assistance-properties}.

\paragraph{Centralized Algorithm.} The reduction lets us apply the Randomized-Augmented OCO (RAOCO) algorithm of \citet{salem2024online} directly. RAOCO reduces online submodular maximization of weighted threshold potentials to online convex optimization.  It runs a standard OCO algorithm (e.g., Follow the Perturbed Leader or Online Mirror Descent) on concave relaxations of the objectives over the convex hull of the strategy space, then rounds each iterate back to an integral point. The weighted threshold and Lipschitz properties of $V_\theta$ ensure that the rounding error is small, yielding the bound in \Cref{prop:centralized-regret}.

The full proof, which combines \Cref{prop:assistance-matroid-reduction} and the structural properties with the regret guarantees of RAOCO (restated in \Cref{prop:raoco-regret}), is deferred to \Cref{proof:centralized-regret}.

\section{Decentralized Learning in Assistance Games: Stable and Adaptive Algorithms}\label{sec:decentralized}
In this section, we move beyond the idealized centralized setting and consider the decentralized setting. \Cref{lem:stable-adaptive} forms the backbone for building on top of the algorithm design in the centralized setting. It shows how to construct an algorithm for the human $\alg_H$ based on a centralized algorithm $\alg_C$ via the construction in \Cref{alg:CAH}. Additionally, it identifies stability of the centralized algorithm, in the form of having low switches, and adaptivity of the assistant's algorithm $\alg_A$, in the form of having low tracking regret, as sufficient properties for yielding low assistance regret. Directed by this, we will construct a stable centralized algorithm that has low external regret while making few switches in \Cref{sec:stable-centralized} and an adaptive algorithm for the assistant having low tracking regret in \Cref{sec:adaptive-assistant}. In \Cref{sec:main-proof}, we put together these components to prove the regret guarantees provided by our main theorems.

\subsection{The Stable Centralized Algorithm}\label{sec:stable-centralized}
In this section, we design a stable centralized algorithm that makes a small number of switches. We build on the centralized algorithm designed in \Cref{sec:centralized}.

\begin{proposition}
    \label{prop:stable-centralized}
    There is a $\mathrm{poly}(M_H, M_A, N, T)$ time algorithm for the centralized assistance game that achieves expected $(1-\nicefrac{1}{e})$-approximate assistance regret $\bO{M_A M_H \sqrt{T}}$ and makes at most $S_T \in O(\sqrt{M_A M_H T}\log(T/\delta))$ switches with probability at least $1 - \delta$.
\end{proposition}

\begin{proof}[Proof sketch of \Cref{prop:stable-centralized}]
Recall that the centralized algorithm of \Cref{sec:centralized} takes an online convex optimization algorithm and rounds the output of the OCO algorithm to the strategy space of the assistance game. To make this algorithm stable, we account for two sources of switches: switches in the OCO algorithm's iterates and switches in the rounding step.

\paragraph{Stable OCO.} Online convex optimization with few switches has been studied by previous work~\citep{agarwal2024lazy,anava2015memory,sherman2021lazy}. We use the Private Continuous Online Multiplicative Weights with Euclidean Regularization (POMER) algorithm of \citet{agarwal2024lazy} as the inner OCO optimizer, POMER requires only that the losses be convex and Lipschitz, conditions that hold for the value-of-assistance concave relaxation $\tilde{V}_\theta$ (\Cref{lem:assistance-is-Lipschitz}). 

\paragraph{Coupled rounding.} Even if the OCO algorithm's iterates are stable, the rounded outputs may not be, since the rounding algorithm used (like swap rounding or randomized pipage rounding~\citep{chekuri2010dependent}) is randomized. So even though the OCO algorithm returns the same solution in different rounds, the randomness in the rounding can lead to different rounded outputs. To overcome this, we couple the randomness of the rounding step across rounds. That is, the same source of randomness is used for rounding across all rounds. Doing this does not change the marginal distribution of the rounded output. In particular, it does not change the expected reward of the rounded output. So the property on the expected reward of the rounded output that is used in the regret analysis (the property in \Cref{prop:sandwich-wtp}) is preserved. This algorithm is described in \Cref{alg:cr-raoco}.

The full proof, which combines the regret guarantee of the underlying stable OCO algorithm with the switch-preserving property of coupled rounding, is given in \Cref{proof:stable-centralized}.
\end{proof}

\subsection{The Adaptive Assistant Algorithm}\label{sec:adaptive-assistant}
We now design a computationally efficient assistant algorithm with low tracking regret relative to the policies chosen by the stable centralized algorithm. Tracking regret is a well-studied notion in online learning, with standard algorithms achieving $\bO{\sqrt{Tp}}$ regret against any sequence of reward functions over $p$ segments. We lift these into an assistant algorithm with low tracking regret over the assistant policy space $\Pi_A$.

\begin{proposition}\label{prop:black-box-tracking-regret}
    There is a $\mathrm{poly}(M_H, M_A, N, T)$ time assistant algorithm $\alg_A$ whose tracking regret over $p$ segments satisfies $R^{\mathrm{track}}_T(\alg_A;\, p) \;\in\; \bO{M_H \sqrt{M_A T p\log (M_A T)}}$.
\end{proposition}

\begin{proof}[Proof sketch of \Cref{prop:black-box-tracking-regret}]
We construct $\alg_A$ from an off-the-shelf tracking-regret minimizer $\alg_{\mathrm{track}}$ (e.g.\ Fixed-Share~\citep{herbster1998tracking,cesa2012mirror}) used as a black box. The assistant maintains $M_H$ independent copies of $\alg_{\mathrm{track}}$, one per human action $a_H \in \mathcal{A}_H$, each optimizing over the assistant's action space $\mathcal{A}_A$. In round $t$, after observing $a_H^{(t)}$, the assistant plays the action selected by the copy associated with $a_H^{(t)}$ and updates only that copy with the observed reward. The assistant's tracking regret then decomposes as the sum of the tracking regrets of the $M_H$ copies of $\alg_{\mathrm{track}}$, which yields the claimed bound after substituting the standard Fixed-Share tracking-regret guarantee (restated as \Cref{thm:fixed-share} in \Cref{sec:tracking-regret}). The full proof appears in \Cref{proof:black-box-tracking-regret}.

\end{proof}

\subsection{Putting It All Together}\label{sec:main-proof}
We now combine the stable centralized algorithm of \Cref{sec:stable-centralized} and the adaptive assistant algorithm of \Cref{sec:adaptive-assistant} via the stable--adaptive decomposition of \Cref{lem:stable-adaptive} to prove our main theorems. Full proofs are in \Cref{proof:natural-regret,proof:main-algorithm}.

\paragraph{Proof idea for \Cref{prop:natural-regret}.} The decentralized algorithms here are built entirely from general-purpose no-regret components: the human algorithm $\alg_H$ is derived (via \Cref{alg:CAH}) from the stable centralized algorithm $\alg_C$ of \Cref{prop:stable-centralized}, and the assistant runs the tracking-regret algorithm $\alg_A$ of \Cref{prop:black-box-tracking-regret}. Substituting $\alg_C$'s external regret and high-probability switching bound together with $\alg_A$'s tracking regret into the stable--adaptive decomposition (\Cref{lem:stable-adaptive}) yields the claimed $\bO{M_H^{5/4} M_A^{3/4}\log(M_A T)\, T^{3/4} + M_A M_H \sqrt{T}}$ bound; the full calculation appears in \Cref{proof:natural-regret}. Note that this construction, and the assistant's algorithm in particular, requires minimal knowledge of the assistance game structure. The assistant merely needs to know its action space and its observed reward.

\paragraph{Proof idea for \Cref{thm:main-algorithm}.}
We use algorithms more tailored to the assistance game for a $\bOtilde{\sqrt{T}}$ rate, optimal in $T$ up to logarithmic factors. We allow the human and assistant some initial synchronization in the form of a shared mapping $\phi : \mathcal{A}_H^* \to \Pi_A$ from human action sequences to assistant policies, agreed on before the game begins. The mapping $\phi$ lets the human signal a policy switch and encode the new assistant policy via a string of actions. The assistant decodes this string using $\phi$ and then plays the new policy directly, bypassing further exploration and decreasing its tracking regret. Plugging this into the stable--adaptive decomposition (\Cref{lem:stable-adaptive}) yields the claimed $\bOtilde{\sqrt{T}}$ bound. The full proof, including details of how the shared mapping is used, appears in \Cref{proof:main-algorithm}.

\section{Discussion}\label{sec:conclusion}

Our framework identifies two broad components useful for learning to cooperate under information asymmetry. The first is \emph{submodularity} as a source of tractable approximation. This points to the usefulness of methods that exploit submodular structure---such as greedy algorithms that measure the marginal improvement of adding actions to communicate or steer. The second component is \emph{stability} and \emph{adaptivity} as properties of the dynamics of interaction. Our decomposition (\Cref{lem:stable-adaptive}) shows that algorithms satisfying these standard online-learning notions plug in directly as building blocks for learning algorithms in these settings.

\paragraph{Limitations and Extensions.}
Our regret bounds are tight in $T$ but the dependence on the action-space sizes $M_H$ and $M_A$ may be loose; determining the optimal scaling is an open question. The optimal $\bO{\sqrt T}$ rate also relies on an initial shared encoding $\phi : \mathcal{A}_H^* \to \Pi_A$, and it is open whether the same rate is achievable without such pre-game synchronization. Finally, it would be interesting to extend the framework to richer interaction settings such as long horizon interactions under a fixed preference state, or two-sided information asymmetry in which the assistant also holds private information.

\section*{Acknowledgments}
 This work was supported in part by the National Science Foundation under grant CCF-2145898, by the Office of Naval Research under grant N00014-24-1-2159, by a Google Research Scholar Award, an Alfred P.~Sloan fellowship, and a Schmidt Science AI2050 fellowship, and by a gift from Coefficient Giving (formerly Open Philanthropy) to support the work of the Center for Human-Compatible AI at Berkeley. We also wish to acknowledge funding by the European Union (ERC-2022-SYG-OCEAN-101071601). Views and opinions expressed are however those of the author(s) only and do not necessarily reflect those of the European Union or the European Research Council Executive Agency. Neither the European Union nor the granting authority can be held responsible for them.

\bibliographystyle{plainnat}
\bibliography{ref}

\newpage
\appendix

\section{Extended Model and Preliminaries}
\label{app:prelim-extras}

This appendix presents the two special subclasses of submodular functions and matroids that arise in our reduction of assistance games, together with the concave relaxation we use to optimize over them. We also state the formal definition of the centralized online assistance game referenced in \Cref{sec:centralized}.

\paragraph{Centralized online assistance game.}

\begin{definition}[Centralized online assistance game]\label{def:central_communication}
    The centralized online assistance game is a repeated game between nature and a meta-player who controls both the assistant and human. At every round $t$, both players take actions simultaneously.\footnote{This means the meta-player chooses a pair of actions for the human and assistant without having access to the realized state. We do this to make the idealized setting serve as a building block for the decentralized setting where the assistant must choose a policy without having access to the realized state.} That is, at the start of the interaction, nature chooses a sequence of preferences $\theta^{(1)}, \cdots, \theta^{(T)} \in \Theta$. At every iteration, the meta-player chooses a pair of policies $\left(\pi_H^{(t)}, \pi_A^{(t)}\right)$, then the realized preference $\theta^{(t)}$ is revealed to the meta-player, and the meta-player receives a reward $r\left(\pi_H^{(t)}, \pi_A^{(t)}; \theta^{(t)}\right)$.
\end{definition}

\paragraph{Weighted threshold potentials.}
A useful subclass of submodular functions consists of those that can be written as a sum of capped linear functions.

\begin{definition}[Weighted Threshold Potential]
    \label{def:weighted-coverage}
    Fix a set of elements $\mathcal{U}$.
    A set function $g$ over $\mathcal{U}$ is \textit{modular} when each element $i\in \mathcal{U}$ is associated with a weight $w_i$ such that $g(S) = \sum_{j \in S} w_{j}$.
    A set function $f$ over $\mathcal{U}$ is a \textit{threshold potential function}, also known as a \textit{budget-additive} set function, when
    \begin{equation*}
        f(S) = \min\{b, g(S)\}
    \end{equation*}
    for some $b \in \mathbb{R}$ and some monotone modular function $g$.
    A set function $f$ over $\mathcal{U}$ is a \textit{weighted threshold potential} when it is a non-negative linear combination of threshold potential functions.
\end{definition}

All weighted threshold potentials are submodular~\citep{salem2024online}.

\paragraph{Partition matroids.}
A useful subclass of matroids in our setting is the partition matroid: the set of policies in an assistance game forms a matroid of this form.

\begin{definition}[Partition Matroids]
    \label{def:partition-matroids}
    A matroid $(\mathcal{U}, \mathcal{I})$ is a \textit{partition matroid} when there exists some $k \in \mathbb{N}$, a partition of $\mathcal{U}$ into subsets $\mathcal{U}_1, \dots, \mathcal{U}_k$, and integers $d_1, \dots, d_k$ such that a set $I$ is independent if and only if
    \begin{equation*}
        \abs{I \cap \mathcal{U}_i} \leq d_i.
    \end{equation*}
\end{definition}

\paragraph{Concave relaxation.}
To optimize over the space of weighted threshold potential functions, we follow \citet{salem2024online}, which runs an online convex optimization algorithm on a concave relaxation of these functions.

\begin{definition}[Concave relaxation of weighted threshold potentials]
    \label{def:concave-relaxation}
    Given a weighted threshold potential $f(S) = \sum_{\ell=1}^L c_\ell \min\{b_\ell, \sum_{j \in S} w_{\ell,j}\}$ defined over some $\mathcal{I} \subseteq 2^{\mathcal{U}}$, its \emph{concave relaxation} $\tilde{f}: [0,1]^{|\mathcal{U}|} \to \mathbb{R}$ is
    \begin{equation*}
        \tilde{f}(\mathbf{y}) = \sum_{\ell=1}^L c_\ell \min\left\{b_\ell,~ \sum_{j \in \mathcal{U}} w_{\ell,j}\, y_j\right\}.
    \end{equation*}
    This function is concave, since each $\min\{b_\ell, \cdot\}$ applied to an affine function is concave, and a positive linear combination of concave functions is concave. Moreover, $\tilde{f}$ agrees with $f$ on integral points: $\tilde{f}(\mathds{1}_S) = f(S)$ for all $S \in \mathcal{I}$.
\end{definition}

\section{An Adaptive Nature Player Is Too Powerful}\label{sec:adaptive-adversary}

A key assumption in this work is that nature cannot adaptively pick preferences based on the choices of the human and assistant. Nature becomes too powerful otherwise, removing any hope the learners have of achieving sub-linear regret. This is analogous to results in the prediction from experts with switching costs literature, where the learner must suffer linear regret with an adaptive adversary \citep{altschuler2018online}.

Nature's strategy will be to exploit the fact that the human and assistant play moves that are uncorrelated given the history of the game, which nature also has access to. By being adaptive, nature can exploit uncoordinated agents by taking advantage of the fact that it is impossible to randomize over best-response pairs in a way that beats an adaptive nature player when actions are independently chosen.

We will prove an adaptive nature player is too powerful, even for a very simple and restricted class of assistance games.
Consider the class of games where the assistant's action set equals the set of preferences, and the agents only get a reward when the assistant can reproduce the human's preference. The human's action has no effect on the reward, and serves only to signal the preference to the assistant. Even on such simple games, when nature can adversarially choose rewards, the agents must suffer linear regret.

First, we prove a lower bound when the human has one less action than the number of preferences, so there does not exist a perfect signaling scheme the human can use to exactly reveal their preferences.

\begin{lemma}
    \label{lem:adaptive-too-powerful}
    When nature is adaptive, there exist assistance games where $N \geq 3$ and $M_H = N - 1$, such that any learning algorithm must achieve a regret of at least $\frac{N - 2}{2N}\cdot T$.
\end{lemma}
\begin{proof}
    Consider the following assistance game. The set of preferences is an arbitrary set $\Theta$ of size $N$, the set of human actions is an arbitrary set $\mathcal{A}_H$ of size $N - 1$, and the set of assistant actions is the set of preferences $\mathcal{A}_A = \Theta$. The reward function is defined as $r(a_H, \theta' ; \theta) = \mathds{1}(\theta = \theta')$.

    First we will lower bound the utility of the optimal human and assistant strategy in hindsight. Notice that regardless of the preferences chosen by nature, there must always exist a human-assistant policy pair that achieves $\frac{N - 1}{N}T$ utility in hindsight. Indeed, let $\theta^{(1)}, \dots, \theta^{(T)}$ be the sequence of preferences played by the adversary. There must exist a preference $\theta$ that appears at most $\frac{T}{N}$ times. The rest of the states must appear at least $\frac{N - 1}{N}\cdot T$ times. Let $\pi_H$ be the human policy that assigns to each of these preferences a unique one of the $N - 1$ messages, and $\pi_A$ be the robot policy that reverses this. This correctly recovers the state when $\theta$ does not appear, so achieves a utility of at least $\frac{N - 1}{N}\cdot T$.

    Now we will lower bound the utility the agents can achieve during the learning process. On step $t$ over the $T$ steps, let $\mathcal{H}_t = \left(\theta^{(1)}, \pi_H^{(1)}, \pi_A^{(1)}, \dots, \theta^{(t - 1)}, \pi_H^{(t - 1)}, \pi_A^{(t - 1)}\right)$ be the history of preferences, human policies, and assistant policies played. The key property that we will exploit is that the policies the human and assistant play are independent given the history, so that $\pi_H^{(t)}~\bot~ \pi_A^{(t)} \mid \mathcal{H}_t$.

    Given a $\pi_H^{(t - 1)}$ and $\pi_R^{(t - 1)}$, what preference should the adversary play? The expected utility that the human and assistant achieve, when the adversary plays state $\theta \in \Theta$ is the probability that the agents correctly recover $\theta$:
    \begin{align*}
        \bP_{\pi_H^{(t)}, \pi_A^{(t)}}\left(\theta = \pi_A^{(t)}\left(\pi_H^{(t)}(\theta)\right) \mid \mathcal{H}_t\right)
        &= \sum_{a_H \in \mathcal{A}_H} \bP_{\pi_H^{(t)}, \pi_A^{(t)}}\left(\theta = \pi_A^{(t)}(a_H) \text{ and } a_H = \pi_H^{(t)}(\theta) \mid \mathcal{H}_t\right) \\
        &= \sum_{a_H \in \mathcal{A}_H} \bP_{\pi_A^{(t)}}\left(\theta = \pi_A^{(t)}(a_H)\mid\mathcal{H}_t\right) \bP_{\pi_H^{(t)}}\left(a_H = \pi_H^{(t)}(\theta) \mid \mathcal{H}_t\right)\\
        &= \E_{a_H \sim \pi_H^{(t)}(\theta)\mid \mathcal{H}_t}\bigg[\bP_{\pi_A^{(t)}}\left(\theta = \pi_A^{(t)}(a_H) \mid \mathcal{H}_t\right)\bigg].
    \end{align*}
    We will show that there must always exist a preference $\theta \in \Theta$ such that
    \begin{equation*}
        \E_{a_H \sim \pi_H^{(t)}(\theta)\mid \mathcal{H}_t}\bigg[\bP_{\pi_A^{(t)}}\left(\theta = \pi_A^{(t)}(a_H) \mid \mathcal{H}_t\right)\bigg] \leq \frac{1}{2}.
    \end{equation*}
    Order the preferences $\theta_1, \dots, \theta_N$. If this is true for any of the first $N - 1$ preferences, we are done. Otherwise, it is the case that for all $i$ from 1 to $N - 1$,
    \begin{equation*}
        \E_{a_H \sim \pi_H^{(t)}(\theta)\mid \mathcal{H}_t}\bigg[\bP_{\pi_A^{(t)}}\left(\theta_i = \pi_A^{(t)}(a_H) \mid \mathcal{H}_t\right)\bigg] > \frac{1}{2}.
    \end{equation*}
    Therefore, for each $i$ from 1 to $N - 1$, there must exist some $a_{H, i} \in \mathcal{A}_H$ such that
    \begin{equation}
        \label{eq:adaptive-upper}
        \bP_{\pi_A^{(t)}}\left(\theta_i = \pi_A^{(t)}(a_{H, i}) \mid \mathcal{H}_t\right) > \frac{1}{2}.
    \end{equation}
    For any $i$ from 1 to $n$, for any preference $\theta \in \Theta$ with $\theta \neq \theta_i$:
    \begin{align*}
        \bP_{\pi_A^{(t)}}\left(\theta = \pi_A^{(t)}(a_{H, i}) \mid \mathcal{H}_t\right) &= 1 - \bP_{\pi_A^{(t)}}\left(\theta \neq \pi_A^{(t)}(a_{H, i}) \mid \mathcal{H}_t\right) \\
        &\leq 1 - \bP_{\pi_A^{(t)}}\left(\theta_i = \pi_A^{(t)}(a_{H, i}) \mid \mathcal{H}_t\right) \\
        &< 1 - \frac{1}{2} = \frac{1}{2}.
    \end{align*}
    So, for any preference $\theta$ that is not $\theta_i$,
    \begin{equation}
        \label{eq:adaptive-lower}
        \bP_{\pi_A^{(t)}}\left(\theta = \pi_A^{(t)}(a_{H, i}) \mid \mathcal{H}_t\right) < \frac{1}{2}.
    \end{equation}
    By Equations (\ref{eq:adaptive-upper}) and (\ref{eq:adaptive-lower}), no two $a_{H, i}$ can be equal. Since this is true for $N - 1 = M_H$ preferences, and every preference is assigned a unique human action by the above, every human action appears in the set of $a_{H, i}, \dots, a_{H, N - 1}$. That is, there is a perfect matching between the set of all human actions and the first $N - 1$ preferences. When an action $a_{H, i}$ is played, it must correspond to some preference $\theta_i$ with high probability. Therefore, for the final preference $\theta_N$, every action is unlikely to correspond to it, i.e. for every action $a_H \in \mathcal{A}_H$,
    \begin{equation*}
        \bP_{\pi_A^{(t)}}\left(\theta_N = \pi_A^{(t)}(a_{H}) \mid \mathcal{H}_t\right) \leq \frac{1}{2},
    \end{equation*}
    and so
    \begin{equation*}
        \E_{a_H \sim \pi_H^{(t)}(\theta)\mid \mathcal{H}_t}\bigg[\bP_{\pi_A^{(t)}}\left(\theta_N = \pi_A^{(t)}(a_H) \mid \mathcal{H}_t\right)\bigg] \leq \frac{1}{2}.
    \end{equation*}

    So, the adversary can always play a preference $\theta$ that forces the expected utility the agents achieve to be at most $\frac{1}{2}$. Whereas in hindsight, the agents could have achieved an average of $\frac{N - 1}{N}$ per step, resulting in an expected regret of at least:
    \begin{equation*}
        \E\left[\max_{\pi_H, \pi_A} \sum_{t = 1}^T \mathbbm{1}(\theta^{(t)} = \pi_A(\pi_H(\theta^{(t)})) - \sum_{t = 1}^T \mathbbm{1}\left(\theta^{(t)} = \pi_A^{(t)}\left(\pi_H^{(t)}(\theta^{(t)})\right)\right)\right]
        \geq \frac{N - 1}{N}\cdot T - \frac{1}{2}\cdot T 
        = \frac{N - 2}{2N}\cdot T. \qedhere
    \end{equation*}
\end{proof}

To extend this result to the case where the number of messages is arbitrary, the adversary can commit beforehand to only using a subset of the states. It is interesting that even if the learners know which subset of the states the adversary has committed to using, they still cannot achieve sublinear regret.

\begin{theorem}
    \label{thm:adaptive-too-powerful}
    When nature is adaptive, there exist assistance games with $N \geq 3$ and $1 < M_H < N$, such that any learning algorithm must suffer a regret of at least $\frac{M_H - 1}{2M_H + 2}T$.
\end{theorem}
\begin{proof}
    When $M_H$ may be arbitrary, the adversary may just commit upfront to only sending $M_H + 1 \leq N$ states, fixed at the start arbitrarily. Since $M_H \geq 2$, this means that there are $M_H + 1 \geq 3$ states, so the lower bound in \Cref{lem:adaptive-too-powerful} applies directly.
\end{proof}

For the learning algorithms we derive for assistance games to be nontrivial, we then must assume that the adversary is oblivious to the actions of the human and assistant.

\section{A Lower Bound on Assistance Regret}\label{sec:lower-bound}

In this section we prove \Cref{thm:no-efficient}, that efficiently achieving an approximation ratio better than $1 - \nicefrac{1}{e}$ is impossible unless $\textrm{P} = \textrm{NP}$.

First, it will be helpful to consider the offline case, where preferences are drawn from a distribution $\mathcal{D}$ over preferences, and a human and assistant policy pair $\pi_H, \pi_A$ must be outputted to maximize the expected reward when preferences are drawn from $\mathcal{D}$:
\begin{equation*}
    \E_{\theta \sim \mathcal{D}}[r(\pi_H(\theta), \pi_A(\pi_H(\theta)); \theta)].
\end{equation*}

We start by showing that achieving an approximation better than $1 - \nicefrac{1}{e}$ in the offline problem is NP-hard.

\begin{restatable}[Assistance is NP-hard]{lemma}{NPComplete}
    \label{lem:NP-Complete}
    Computing any (potentially stochastic) human-assistant policy pair that is an $\alpha$-approximation of the optimal strategy, where $\alpha > 1 - \nicefrac{1}{e}$, in an offline assistance game, even with rewards restricted to be 0 or 1, is NP-hard.
\end{restatable}
\begin{proof}
    We will in fact show this offline hardness result for games where the human's action has no effect on utility.

    Recall that the problem of finding $\alpha$-approximations to the max $k$-coverage problem with $\alpha > 1 - \nicefrac{1}{e}$ is NP-hard \citep{feige1998threshold}. Then, we will reduce this to finding an $\alpha$-approximation of optimal play in an assistance game. The rewards in this game will be 0 or 1 only.

    We start by recapping the $k$-maximum coverage problem. Given a value $k$ and a collection of sets $S = \{S_1, \dots, S_m\}$ each with elements in a universe $\mathcal{U}$, find a collection of $k$ sets $S_{i_1}, \dots, S_{i_k}$ such that $\abs{\bigcup_{\ell = 1}^k S_{i_\ell}}$ is maximized. \citet{feige1998threshold} proved that it is NP-hard to get $(1 - \nicefrac{1}{e})$-approximations of the maximum $k$ coverage of $S$.

    We now describe a reduction from a maximum coverage problem to solving offline assistance games. Consider an arbitrary $k$-maximum coverage problem over a collection of sets $S = \{S_1, \dots, S_m\}$ with elements in $\mathcal{U}$. Create an assistance game where the human's action set is an arbitrary set of size $M_H = k$, and the space of preferences is $\mathcal{U}$. Make the assistant's action space $S$, so that each $S_i$ is an assistant action, Set the reward function to
    \begin{equation*}
        r(a_H, S_i; e) = \begin{cases}
            1 & \text{if $e \in S_i$,} \\
            0 & \text{otherwise.}
        \end{cases}
    \end{equation*}
    The human's action has no effect on utility and serves only to signal the preference.

    Let the underlying distribution $\mathcal{D}$ assign an equal probability to each state, $1/\abs{\mathcal{U}}$.

    We will prove two useful facts about this game. \begin{enumerate}
        \item Every solution to the max-$M$-coverage problem induces a solution to this assistance game with the same value. More precisely, every collection of sets $S_{i_1}, \dots, S_{i_M}$, with union $S = \bigcup_{j = 1}^M S_{i_j}$, induces an human-assistant policy pair $\pi_H, \pi_A$ that achieves utility that is at least the value of the sets in the maximum-$M$-coverage problem: $\abs{S}/\abs{\mathcal{U}}$. Indeed, let the human sends message $m_j$ when observing a state in $S_{i_j} \setminus \bigcup_{\ell = 1}^{j - 1} S_{i_\ell}$. The assistant plays set $S_{i_j}$ upon receiving message $m_j$. Let $S = \bigcup_{j = 1}^M S_i$. Whenever a preference in $S$ appears, the agents get a reward of 1, and so the expected reward of this human-assistant policy pair is $\abs{S}/\abs{\mathcal{U}}$, precisely the value of the sets in the maximum-$M$-coverage problem.

        \item Every solution with a deterministic assistant policy to this assistance game induces a solution to the max-$M$-coverage problem with at least the same value. More precisely, given a deterministic assistant policy $\pi_A$, we will construct a collection of sets $S_{i_1}, \dots, S_{i_M}$, with union $S = \bigcup_{j = 1}^M S_{i_j}$, so that for any human policy $\pi_H$, the utility $\pi_H, \pi_A$ achieves utility is at most the value of the sets in the maximum-$M$-coverage problem: $\abs{S}/\abs{\mathcal{U}}$.
        
        $\pi_A$ maps each message to an action $S_i$. There are $M_H$ possible messages, so simply take the $M_H$ actions, $S_{i_1}, \dots, S_{i_M}$, that the assistant chooses to play as the solution to the maximum coverage problem. Let $S = \bigcup_{j=1}^M S_{i_j}$. The value $\pi_A$ achieves with any human policy $\pi_H$ is at most the probability that elements in $S$ appear: $\abs{S}/\abs{\mathcal{U}}$.
    \end{enumerate}
    
    These reductions back and forth imply that the optimal values of both problems are equal.

    In general assistance games, given a human policy $\pi_H$ that is not necessarily deterministic, there exists a deterministic best-response assistant policy $\pi_A^*$ that can be found by an efficient algorithm. Indeed, given a human action $a_H \in \mathcal{A}_H$ the value of playing action $a_A \in \acts_A$ for the assistant is
    \begin{align*}
        \E_{\theta \sim \mathcal{D}}[r(a_H, a_A, \theta) \mid \pi_H(\theta) = a_H] &= \sum_{\theta \in \Theta} \bP(\theta \mid \pi_H(\theta) = a_H)r(a_H, a_A, \theta) \\
        &= \sum_{\theta \in \Theta} \frac{\bP(\theta)\bP(\pi_H(\theta) = a_H)}{\sum_{\theta' \in \Theta} \bP(\theta')\bP(\pi_H(\theta') = m)} r(a_H, a_A, \theta),
    \end{align*}
    which can be computed in polynomial time in the number of actions and states. The reward any response achieves is linear, and the optimal deterministically plays the action with largest expected reward given the human action observed.

    Using these two facts, we can reduce the problem of finding an $\alpha$-approximation of the $k$-maximum coverage problem to finding an $\alpha$-approximation of the optimal human-assistant policy pair in this assistance game. Suppose $(\pi_H,\pi_A)$ are a human-assistant policy pair that achieve an $\alpha$-approximation of optimal play in this assistance game. $(\pi_H, \pi_R)$ may not necessarily be deterministic. But, by the above, we can construct in polynomial time a deterministic assistant policy $\pi_A'$ that achieves at least the same reward when paired with $\pi_H$.
    
    Applying fact 2 on $\pi_A'$ this means that we have found a collection of sets $S_{i_1}, \dots, S_{i_k}$ that achieve at least $\alpha$ times the max reward of the assistance game. So we have shown that the optimal reward in this game is equal to the optimal reward in the maximum-$k$-coverage problem, and we are done!
\end{proof}

With this tool in hand, we can show that, unless $\textrm{RP} = \textrm{NP}$, efficient learning algorithms for assistance games that achieve approximate utility better than $1 - \nicefrac{1}{e}$ cannot exist.

\NoEfficient*

\begin{proof}[Proof of \Cref{thm:no-efficient}]
    Suppose for the sake of contradiction that $R_T \in \bO{T^{1 - \epsilon}}$ for some $\epsilon > 0$.

    We can write for some constants $c_1, c_2, a, b, c \in \mathbb{R}$, that
    $R_T(N, M) \leq c_1 N^a M_H^b M_A^c T^{1 - \epsilon} + c_2$.
    This means that the average regret per time step is at most
    \begin{equation*}
        \frac{c_1N^a M_H^b M_A^c}{T^\epsilon} + \frac{c_2}{T}.
    \end{equation*}

    The rest of the proof will follow the lead of \citet{kapralov2013online}. Given a value $k \in \mathbb{N}$ and a collection of sets $S = \{S_1, \dots, S_m\}$ each with elements in a universe $\mathcal{U}$, \citet{feige1998threshold} proved that, for any $\delta > 0$, it is NP-hard to distinguish between the following two cases:
    \begin{enumerate}
        \item Yes case: Some choice of $k$ sets covers $\mathcal{U}$.
        \item No case: No choice of $k$ sets covers a $1 - \nicefrac{1}{e} + \delta$ proportion of the elements in $\mathcal{U}$.
    \end{enumerate}
    Therefore, if there is a polynomial-time algorithm that outputs yes with constant probability when in the yes case, and outputs no always when in the no case, it must be that $\text{RP} = \text{NP}$.

    Suppose there exists an efficient no-$\alpha$-approximate regret assistance game learning algorithm. We will show that the problem above for $\delta = \frac{\alpha - (1 - \nicefrac{1}{e})}{2}$ can be solved with a constant probability of success in the yes case, and always in the no case. This would prove $\text{RP} = \text{NP}$.

    Given an instance of the set-cover problem,
    run the algorithm on the corresponding assistance game derived in \Cref{lem:NP-Complete}. As we show in \Cref{lem:NP-Complete}, the optimal value in this game is equal to that of the set-cover problem. 
    
    Run the no-$\alpha$-approximate regret learning algorithm for
    \begin{equation*}
        T = \max\left(c_2,~ \left(\left(\frac{\alpha - (1 + \nicefrac{1}{e})}{4}\right)c_1N^aM_H^bM_A^c\right)^{\nicefrac{1}{\epsilon}}\right) \in \text{poly}(N, M)
    \end{equation*}
    steps, and the average regret per iteration becomes at most $\frac{\alpha - (1 - \nicefrac{1}{e})}{4}$.

    First, notice that we may as well pretend that the human and assistant output their full policy at every step $\pi_H^{(t)}$ and $\pi_A^{(t)}$ at every step, not just $\pi_H^{(t)}(\theta^{(t)})$ and $\pi_A^{(t)}(\pi_H^{(t)}(\theta^{(t)}))$. This is because at every step of the learning algorithm, we can create $M_H$ and $M_A$ copies and ask for the distribution over actions on every counterfactual input.

    The reduction will be as follows. Run the algorithm on $T$ i.i.d draws $\theta^{(1)}, \dots, \theta^{(T)} \sim \mathcal{D}^T$, induce a sequence of human-assistant policy pairs $(\pi_H^{(1)}, \pi_A^{(1)}), \dots, (\pi_H^{(T)}, \pi_A^{(T)})$, compute the expected reward of each human-assistant policy pair averaged over $\mathcal{D}$ (which we can do in at most $\bO{N}$ time each), and output yes if some pair $\pi_H^{(t)}, \pi_A^{(t)}$ exists with expected reward at least $\alpha - \frac{3(\alpha - (1 - \nicefrac{1}{e}))}{8}$, and no otherwise.

    When we are in the no case, no solution will have expected reward in the assistance game that is at least $\alpha - \frac{3(\alpha - (1 - \nicefrac{1}{e}))}{8}$, so we will always output no.

    Suppose we are in the yes case, so that the optimal utility is 1. Because we run for long enough that average regret is $\frac{\alpha - (1 - \nicefrac{1}{e})}{4}$, we can say
    \begin{equation}
        \label{eq:no-regret-doesnt-exist}
        \E_{\theta^{1:T} \sim\mathcal{D}^T}\left[\frac{1}{T}\sum_{t = 1}^T r\left(\pi_H^{(t)}(\theta^{(t)}),~ \pi_A^{(t)}(\pi_H^{(t)}(\theta^{(t)}));~ \theta^{(t)}\right)\right] \geq \alpha - \frac{\alpha - (1 - \nicefrac{1}{e})}{4}.
    \end{equation}
    To move forward, we will need to remove the dependence on $\theta^{(t)}$ in the expected reward. Notice that $\theta^{(t)} \bot (\pi_H^{(t)}, \pi_A^{(t)})$, as $\theta^{(t)}$ is sampled from $\mathcal{D}$, and $\pi_H^{(t)}, \pi_A^{(t)}$ depends only on $\theta^{(1)}, \dots, \theta^{(t - 1)}$, which themselves are independent of $\theta^{(t)}$.
    Therefore, for any $t$, it must be that
    \begin{align}
        &\E_{\theta^{1:T} \sim\mathcal{D}^T}\left[r\left(\pi_H^{(t)}(\theta^{(t)}),~ \pi_A^{(t)}(\pi_H^{(t)}(\theta^{(t)}));~ \theta^{(t)}\right)\right]\notag\\
        &= \E_{\theta^{1:(t-1)} \sim\mathcal{D}^{(t-1)}}\E_{\theta^{(t)} \sim \mathcal{D}}\left[r\left(\pi_H^{(t)}(\theta^{(t)}),~ \pi_A^{(t)}(\pi_H^{(t)}(\theta^{(t)}));~ \theta^{(t)}\right)\mid \theta^{(1)}, \dots, \theta^{(t - 1)}\right] \notag\\
        &= \E_{\theta^{1:(t-1)} \sim\mathcal{D}^{(t-1)}}\E_{\theta \sim \mathcal{D}}\left[r\left(\pi_H^{(t)}(\theta),~ \pi_A^{(t)}(\pi_H^{(t)}(\theta));~ \theta\right)\right]. \label{eq:lower-bound-sub-expected}
    \end{align}
    And so by the linearity of expectation, 
    \begin{align*}
        &\E_{\theta^{1:T} \sim\mathcal{D}^T}\left[\frac{1}{T}\sum_{t = 1}^T r\left(\pi_H^{(t)}(\theta^{(t)}),~ \pi_A^{(t)}(\pi_H^{(t)}(\theta^{(t)}));~ \theta^{(t)}\right)\right] \\
        &= \frac{1}{T}\sum_{t = 1}^T \E_{\theta^{1:T} \sim\mathcal{D}^T}\left[r\left(\pi_H^{(t)}(\theta^{(t)}),~ \pi_A^{(t)}(\pi_H^{(t)}(\theta^{(t)}));~ \theta^{(t)}\right)\right] \\
        &= \frac{1}{T}\sum_{t = 1}^T \E_{\theta^{1:(t-1)} \sim\mathcal{D}^{(t-1)}}\E_{\theta \sim \mathcal{D}}\left[r\left(\pi_H^{(t)}(\theta),~ \pi_A^{(t)}(\pi_H^{(t)}(\theta));~ \theta\right)\right] && (\text{By \Cref{eq:lower-bound-sub-expected}}) \\
        &= \frac{1}{T}\sum_{t = 1}^T \E_{\theta^{1:T} \sim\mathcal{D}^{T}}\E_{\theta \sim \mathcal{D}}\left[r\left(\pi_H^{(t)}(\theta),~ \pi_A^{(t)}(\pi_H^{(t)}(\theta));~ \theta\right)\right] \\
        &= \E_{\theta^{1:T} \sim\mathcal{D}^{T}}\left[\frac{1}{T}\sum_{t = 1}^T \E_{\theta \sim \mathcal{D}}\left[r\left(\pi_H^{(t)}(\theta),~ \pi_A^{(t)}(\pi_H^{(t)}(\theta));~ \theta\right)\right]\right]
        \end{align*}
    Plugging this into (\ref{eq:no-regret-doesnt-exist}),
    \begin{equation*}
        \E_{\theta^{1:T} \sim\mathcal{D}^{T}}\left[\frac{1}{T}\sum_{t = 1}^T \E_{\theta \sim \mathcal{D}}\left[r\left(\pi_H^{(t)}(\theta),~ \pi_A^{(t)}(\pi_H^{(t)}(\theta));~ \theta\right)\right]\right] \geq \alpha - \frac{\alpha - (1 - \nicefrac{1}{e})}{4}.
    \end{equation*}
    Let $\bar X = \frac{1}{T}\sum_{t = 1}^T \E_{\theta \sim \mathcal{D}}\left[r\left(\pi_H^{(t)}(\theta),~ \pi_A^{(t)}(\pi_H^{(t)}(\theta));~ \theta\right)\right]$. Since rewards lie in $[0, 1]$, the random variable $1 - \bar X$ is non-negative, and the displayed inequality above bounds $\E[1 - \bar X] \le 1 - \alpha + \frac{\alpha - (1 - \nicefrac{1}{e})}{4}$. By Markov's inequality applied to $1 - \bar X$,
    \begin{equation*}
        \Pr\!\left[\bar X \,<\, \alpha - \frac{3(\alpha - (1 - \nicefrac{1}{e}))}{8}\right] \;=\; \Pr\!\left[1 - \bar X \,>\, 1 - \alpha + \frac{3(\alpha - (1 - \nicefrac{1}{e}))}{8}\right] \;\le\; \frac{1 - \alpha + \frac{\alpha - (1 - \nicefrac{1}{e})}{4}}{1 - \alpha + \frac{3(\alpha - (1 - \nicefrac{1}{e}))}{8}},
    \end{equation*}
    and the right-hand side is a constant strictly less than $1$ for any fixed $\alpha > 1 - \nicefrac{1}{e}$. Therefore, with constant probability,
    \begin{equation*}
       \frac{1}{T}\sum_{t = 1}^T \E_{\theta \sim \mathcal{D}}\left[r\left(\pi_H^{(t)}(\theta),~ \pi_A^{(t)}(\pi_H^{(t)}(\theta));~ \theta\right)\right] \geq \alpha - \frac{3(\alpha - (1 - \nicefrac{1}{e}))}{8}.
    \end{equation*}
    Because it is an average, we can say in this case that there must exist some $t$ such that $$\E_{\theta \sim \mathcal{D}}\left[r\left(\pi_H^{(t)}(\theta),~ \pi_A^{(t)}(\pi_H^{(t)}(\theta));~ \theta\right)\right] \geq \alpha - \frac{3(\alpha - (1 - \nicefrac{1}{e}))}{8},$$ and our reduction will find it.

    Thus, we can efficiently determine whether we are in the yes case with constant probability, and we are done!
\end{proof}

\section{Missing proofs}

\subsection{Proof of \Cref{lem:assistance-partition-matroid}}\label{proof:assistance-partition-matroid}

\begin{lemma}[Assistance Matroid is a Partition Matroid, restated]
    \label{lem:assistance-partition-matroid}
    For any assistance game $G$, the corresponding assistance matroid $\mathcal{M}_G = (\mathcal{U}_G, \mathcal{I}_G)$ defined in \Cref{def:assistance-matroid} is a partition matroid with rank $M_H$ (\Cref{def:partition-matroids}).
\end{lemma}
\begin{proof}
    The maximum size of $S \in \mathcal{I}_G$ is $M_H$ because there can be at most one action pair for each human action. So the rank of $\mathcal{M}_G$ is $M_H$.
    To see why this is a partition matroid, consider partitioning the ground set $\mathcal{U}_G$ into disjoint sets $\mathcal{U}_{a_H} = \{a_H\} \times \mathcal{A}_A$ with capacities $d_{a_H}=1$ for $a_H \in \mathcal{A}_H$. Since no two human actions are in the same independent set, we have $|\mathcal{I}_G \cap \mathcal{U}_{a_H}| = d_{a_H} = 1$. This completes the proof.
\end{proof}

\subsection{Proof of \Cref{prop:assistance-matroid-reduction}}\label{proof:assistance-matroid-reduction}
    For every independent set $I \in \mathcal{I}_G$, let $\mathcal{A}_H^I = \{a_H \in \mathcal{A}_H \mid \exists a_A \in \mathcal{A}_A : (a_H, a_A) \in I\}$ be the set of human actions paired with an assistant action in $I$. If $I = \emptyset$, then $V_\theta(I) = 0$, and any policy pair achieves reward at least 0 by definition.

    Otherwise, we can define the associated assistant policy as playing, in response to any human action $a_H \in \mathcal{A}_H^I$ the assistant action paired with $a_H$ in $I$, and in response to any human action $a_H \notin \mathcal{A}_H^I$ an arbitrary assistant action $\overline{a_A} \in \mathcal{A}_A$:
    \begin{equation*}
        \pi_A^I(a_H) = \begin{cases}
            a_A & \text{if $a_H \in \mathcal{A}_H^I$ and $(a_H, a_A) \in I$}, \\
            \overline{a_A} & \text{if $a_H \notin \mathcal{A}_H^I$}.
        \end{cases}
    \end{equation*}
    This is well-defined: because the set $I$ is independent, every human action is paired with at most one assistant action.

    Defining the associated human policy $\pi_H^I: \Theta \to \mathcal{A}_H$ as being the one that plays the optimal response to $\pi_A^I$ with actions in $\mathcal{A}_H^I$,
    \begin{equation*}
        \pi_H^I(\theta) = \argmax_{a_H \in \mathcal{A}_H^I} ~ r(a_H, a_A; \theta),
    \end{equation*}
    the reward the policies $\pi_H^I$ and $\pi_A^I$ achieve when paired together on preference $\theta$ is exactly the value $V_\theta(I)$:
    \begin{align*}
        r(\pi_H^I(\theta), \pi_A^I(\pi_H^I(\theta)); \theta)
        = \max_{a_H \in \mathcal{A}_H^I} r(a_H, \pi_A^I(a_H); \theta)
        = \max_{(a_H, a_A) \in I} r(a_H, a_A; \theta)
        = V_\theta(I).
    \end{align*}

    This proves the first condition. For the second, given any assistant policy $\pi_A \in \Pi_A$ define the set
    \begin{equation*}
        I_{\pi_A} = \{(a_H, \pi_A(a_H)) \in \mathcal{U}_G: a_H \in \mathcal{A}_H\}.
    \end{equation*}
    Because every $a_H$ appears in one pair in $I_{\pi_A}$, $I_{\pi_A}$ is an independent set in $\mathcal{M}_G$. For every preference $\theta \in \Theta$ and every human policy $\pi_H: \Theta \to \mathcal{A}_H$, it follows that
    \begin{align*}
        r(\pi_H(\theta), \pi_A(\pi_H(\theta)); \theta)
        \leq \max_{a_H \in \mathcal{A}_H} r(a_H, \pi_A(a_H); \theta)
        = \max_{a_H, a_A \in I_{\pi_A}} r(a_H, a_A; \theta)
        = V_\theta(I_{\pi_A}).\quad \qedhere
    \end{align*}
\qed

\subsection{Structural properties of the value of assistance function}\label{proof:assistance-properties}

\begin{lemma}[restated]
    \label{lem:assistance-is-weighted-threshold}
    The value of assistance function $V_\theta$ is a weighted threshold potential (\Cref{def:weighted-coverage}).
\end{lemma}
\begin{proof}
    Fix a $\theta \in \Theta$. We will show that $V_\theta$ is a weighted threshold potential.
    Sort the values $r(a_H, a_A; \theta)$ from largest to smallest, $\alpha_1, \dots, \alpha_{M_HM_A}$, where $\alpha_i$ is the $i$th largest reward a human-assistant action pair can achieve, and let $\left(a_H^{(1)}, a_A^{(1)}\right), \dots, \left(a_H^{(M_HM_A)}, a_A^{(M_HM_A)}\right)$ be the corresponding ordering of human-assistant action pairs from highest to lowest. Set $\alpha_{M_HM_A + 1} = 0$.

    The key trick is to rewrite the difference between the max over any non-empty $S$ and the overall minimum over $\mathcal{A}_H \times \mathcal{A}_A$ as a telescoping sum:
    \begin{equation*}
        \max_{(a_H, a_A) \in S} r(a_H, a_A; \theta) = \sum_{k = 1}^{M_HM_A} (\alpha_k - \alpha_{k + 1})\cdot \mathds{1}\left(\exists i\leq k: \left(a_H^{(i)}, a_A^{(i)}\right) \in S\right).
    \end{equation*}
    The indicator function $\mathds{1}\left(\exists i\leq k: \left(a_H^{(i)}, a_A^{(i)}\right) \in S\right)$ is a budget-additive function $$ \mathds{1}\left(\exists i\leq k: \left(a_H^{(i)}, a_A^{(i)}\right) \in S\right) = \min\left\{1,~ \sum_{i \in S} w_i^k \right\},$$
    where $w_{i}^k = \mathds{1}(i \leq k)$. So, defining the modular function $g_k(S) = \sum_{i \in S} w_i^k$, we can write
    \begin{equation*}
        \max_{(a_H, a_A) \in S} r(a_H, a_A; \theta) = \sum_{k = 1}^{M_HM_A} (\alpha_{k} - \alpha_{k + 1})\cdot  \min\left\{1,~ g_k(S)\right\}.
    \end{equation*}

    Because the list of $\alpha_i$ is ordered, it follows that $\alpha_k \geq \alpha_{k + 1}$, and so all the coefficients are non-negative. By definition, $\max_{(a_H, a_A) \in S} r(a_H, a_A; \theta) = V_\theta(S)$. So, all in all, when $S$ is non-empty, we can write the value of assistance function as a non-negative linear combination of weighted threshold potentials
    \begin{align}
        \label{eq:WTP-form}
        V_\theta(S) = \sum_{k = 1}^{M_HM_A} (\alpha_{k} - \alpha_{k + 1})\cdot  \min\left\{1,~ g_k(S)\right\}.
    \end{align}
    When $S = \emptyset$, both sides are 0, so the equality always holds.
\end{proof}

For the value of assistance function $V_\theta$ written in weighted threshold potential form (\Cref{eq:WTP-form}), the concave relaxation over the matroid polytope $\mathcal{Y} = \mathrm{conv}\{\mathds{1}_S : S \in \mathcal{I}_G\}$ is
\begin{align*}
    \tilde{V}_\theta(\mathbf{y}) = \sum_{k = 1}^{M_HM_A} (\alpha_{k} - \alpha_{k+1}) \cdot  \min\left\{1,~ \sum_{i=1}^{k} y_{(a_H^{(i)}, a_A^{(i)})}\right\},
\end{align*}
where $\mathbf{y} \in \mathcal{Y} \subseteq [0,1]^{M_HM_A}$.

\begin{lemma}[restated]
    \label{lem:assistance-is-Lipschitz}
    For any preference $\theta \in \Theta$, the concave relaxation $\tilde{V}_\theta$ of the value of assistance function is $1$-Lipschitz with respect to the $\ell_1$ norm over $\mathcal{Y} = [0,1]^{M_H M_A}$.
\end{lemma}
\begin{proof}
    Consider $\theta$ and the corresponding concave relaxation $\tilde{V}_\theta$ defined in \Cref{def:concave-relaxation}. We can write $\tilde{V}_\theta(\mathbf{y}) = \sum_{k=1}^{M_HM_A} |c_k| \min\{1, \mathbf{w}_k^\top \mathbf{y}\}$ where $|c_k| = |\alpha_{k} - \alpha_{k + 1}|$ and $\mathbf{w}_k \in \{0,1\}^{M_HM_A}$ are coefficients and weight vectors derived in the proof of \Cref{lem:assistance-is-weighted-threshold} where we express the value of assistance functions in the form of weighted threshold potentials. For any $\mathbf{y}, \mathbf{y'} \in \mathcal{Y}$:
    \begin{align*}
        |\tilde{V}_\theta(\mathbf{y}) - \tilde{V}_\theta(\mathbf{y'})| &\leq \sum_{k=1}^{M_HM_A} |c_k| \left|\min\{1, \mathbf{w}_k^\top \mathbf{y}\} - \min\{1, \mathbf{w}_k^\top \mathbf{y'}\}\right| \\
        &\leq \sum_{k=1}^{M_HM_A} |c_k| \left|\mathbf{w}_k^\top (\mathbf{y} - \mathbf{y'})\right| \\
        &\leq \sum_{k=1}^{M_HM_A} |c_k| \cdot \|\mathbf{w}_k\|_\infty \cdot \|\mathbf{y} - \mathbf{y'}\|_1 \\
        &\leq \left(\sum_{k=1}^{M_HM_A} |c_k|\right) \|\mathbf{y} - \mathbf{y'}\|_1.
    \end{align*}
    This shows that the function $\tilde{V}_\theta$ is Lipschitz with Lipschitz constant:
    \[\sum_{k=1}^{M_HM_A} |c_k| = \sum_{k=1}^{M_HM_A} (\alpha_k - \alpha_{k + 1}) = \alpha_1 - \alpha_{M_HM_A+1} = \alpha_1 \le 1.\]
    Where the last inequality holds because rewards are bounded in $[0, 1]$.
\end{proof}

\subsection{Proof of \Cref{prop:centralized-regret}}\label{proof:centralized-regret}
    By \Cref{prop:assistance-matroid-reduction}, the centralized assistance game reduces to online maximization of value of assistance functions over the assistance matroid $\mathcal{M}_G$. The assistance matroid has ground set size $|\mathcal{U}_G| = M_HM_A$ and rank $M_H$ (\Cref{lem:assistance-partition-matroid}). By \Cref{lem:assistance-is-weighted-threshold} and \Cref{lem:assistance-is-Lipschitz}, the value of assistance functions are weighted threshold potentials and their concave relaxations are $1$-Lipschitz. Applying the RAOCO algorithm of \citet{salem2024online} with these parameters yields a $\mathrm{poly}(M_H, M_A, N, T)$ time algorithm with expected $(1-\nicefrac{1}{e})$-approximate regret $\bO{M_H\sqrt{T \log M_A}}$. This is a direct result of the regret analysis done in~\cite{salem2024online} which is restated as \Cref{prop:raoco-regret} in the appendix.
\qed

\subsection{Proof of \Cref{lem:stable-adaptive}}\label{proof:stable-adaptive}
    Let $\left(\overline{\pi}_H^{(t)}, \overline{\pi}_A^{(t)}\right)$ be the policy pair output by $\alg_C$ at round $t$. By our construction of $\alg_H$, the human policy played in round $t$ is $\overline{\pi}_H^{(t)}$.  Let $\pi_A^{(t)}$ denote the assistant policy played by $\alg_A$ in round $t$. The assistance regret of the decentralized pair is
    \begin{equation*}
        R^\alpha_T = \alpha \max_{\pi_H, \pi_A} \sum_{t=1}^T r_t(\pi_H, \pi_A) - \sum_{t=1}^T r_t\left(\overline{\pi}_H^{(t)}, \pi_A^{(t)}\right).
    \end{equation*}
    We decompose this into two terms by adding and subtracting the centralized algorithm's reward:
    \begin{align*}
        R_T &= \underbrace{\alpha \max_{\pi_H, \pi_A} \sum_{t=1}^T r_t(\pi_H, \pi_A) - \sum_{t=1}^T r_t\left(\overline{\pi}_H^{(t)}, \overline{\pi}_A^{(t)}\right)}_{R_{\text{central}}} + \underbrace{\sum_{t=1}^T r_t\left(\overline{\pi}_H^{(t)}, \overline{\pi}_A^{(t)}\right) - \sum_{t=1}^T r_t\left(\overline{\pi}_H^{(t)}, \pi_A^{(t)}\right)}_{R_{\text{coord}}}.
    \end{align*}
    The first term $R_{\text{central}}$ is the external regret of the centralized algorithm $\alg_C$, i.e., $R_T^{\alpha,\textrm{ext}}(\alg_C)$.

    For the second term, let $p$ denote the number of times the sequence $\left(\overline{\pi}_H^{(t)}, \overline{\pi}_A^{(t)}\right)_{t=1}^T$ switches, and let $s_1, \ldots, s_p$ be the time indices at which switches occur (with $s_1 = 1$ and $s_{p+1} = T+1$). Within each segment $[s_i, s_{i+1})$, the centralized algorithm plays the same policy pair. We can bound:
    \begin{align*}
        R_{\text{coord}} &= \sum_{t=1}^T r_t\left(\overline{\pi}_H^{(t)}, \overline{\pi}_A^{(t)}\right) - \sum_{t=1}^T r_t\left(\overline{\pi}_H^{(t)}, \pi_A^{(t)}\right) \\
        &= \sum_{i=1}^p \sum_{t=s_i}^{s_{i+1}-1} r_t\left(\overline{\pi}_H^{(s_i)}, \overline{\pi}_A^{(s_i)}\right) - \sum_{i=1}^p \sum_{t=s_i}^{s_{i+1}-1}  r_t\left(\overline{\pi}_H^{(s_i)}, \pi_A^{(t)}\right) \\
        &\leq \max_{\substack{s_1, \ldots, s_p \\ \pi_A^{(1)}, \ldots, \pi_A^{(p)}}} \sum_{i=1}^p \sum_{t=s_i}^{s_{i+1}-1} \left (r_t\left(\overline{\pi}_H^{(s_i)}, \pi_A^{(i)}\right) - r_t\left(\overline{\pi}_H^{(s_i)}, \pi_A^{(t)}\right) \right )\\
        &\leq R^{\text{track}}_T(\alg_A;~p).
    \end{align*}

    Since $p \leq S_T(\alg_C; \delta)$ with probability at least $1 - \delta$, we have
    \begin{equation*}
        R^\alpha_T \leq R_T^{\alpha, \textrm{ext}}(\alg_C) + R_T^{\text{track}}(\alg_A;~S_T(\alg_C; \delta)) + \delta T,
    \end{equation*}
    where the $\delta T$ term accounts for the event that the number of switches exceeds the high-probability bound.
\qed

\subsection{Proof of \Cref{prop:stable-centralized}}\label{proof:stable-centralized}
    From~\Cref{prop:raoco-regret}, the $(1-1/e)$-approximate regret of the CR-RAOCO algorithm is bounded by $(1-1/e)$ times the regret of the underlying OCO algorithm. By \Cref{prop:cr-raoco-switches}, the number of switches of CR-RAOCO is at most the number of switches of the underlying OCO algorithm, since the coupled randomness makes the rounding step a deterministic function of the OCO iterate.

    We instantiate the inner OCO algorithm with the POMER algorithm of \citet{agarwal2024lazy} (\Cref{alg:p-comw}). The convex domain is $\mathcal{Y} = [0,1]^{M_A M_H}$, with diameter $D = \sqrt{M_A M_H}$ and dimension $d = M_A M_H$. By \Cref{lem:assistance-is-Lipschitz}, the concave relaxation $\tilde{V}_\theta$ is $1$-Lipschitz with respect to $\|\cdot\|_1$, hence $G$-Lipschitz with respect to $\|\cdot\|_2$ for $G \le \sqrt{M_A M_H}$. Setting the switching budget $S = \sqrt{M_A M_H T}\,\log(T)$ in \Cref{thm:p-comw} gives
    \[
        R_T \;\le\; GD\sqrt{2T} \;+\; 16\,GD\log(T) \cdot \frac{\sqrt{d}\cdot T}{S} \;+\; 13\,GD \;\in\; O\!\bigl(M_A M_H \sqrt{T}\bigr),
    \]
    where the last bound substitutes $G = D = \sqrt{M_A M_H}$, $d = M_A M_H$, and $S = \sqrt{M_A M_H T}\,\log(T)$, under which the middle term equals $16\, M_A M_H \sqrt{T}$. Multiplying by the $(1-\nicefrac{1}{e})$ factor from \Cref{prop:raoco-regret} preserves the rate, giving the expected $(1-\nicefrac{1}{e})$-approximate assistance regret bound stated in \Cref{prop:stable-centralized}.

    For the switching bound, \Cref{thm:p-comw} gives $\Pr[S_T \ge 3 S] \le e^{-S}$. Choosing $S = \max\!\bigl(\sqrt{M_A M_H T}\,\log(T),~ \log(1/\delta)\bigr)$ yields, with probability at least $1 - \delta$, $S_T \le 3 S \in O\!\bigl(\sqrt{M_A M_H T}\,\log(T) + \log(1/\delta)\bigr) \subseteq O\!\bigl(\sqrt{M_A M_H T}\,\log(T/\delta)\bigr)$, matching the bound stated in \Cref{prop:stable-centralized}. By \Cref{prop:cr-raoco-switches} this same bound applies to the number of switches made by CR-RAOCO.

    Crucially, \Cref{thm:p-comw} requires only convexity and Lipschitz-ness of the losses---both satisfied by $\tilde{V}_\theta$---and does not require the smoothness assumption needed by FTPRLL~\citep{sherman2021lazy}, which fails for the piecewise-linear value-of-assistance relaxation.
\qed

\subsection{Proof of \Cref{prop:black-box-tracking-regret}}\label{proof:black-box-tracking-regret}
    We construct $\alg_A$ using an algorithm $\alg_{\mathrm{track}}$ as a subroutine in a blackbox manner.  We will choose $\alg_{\mathrm{track}}$ to be an algorithm designed to minimize tracking regret in general online learning problems (such as algorithms from~\cite{herbster1998tracking, cesa2012mirror}).

    We will first describe how $\alg_A$ is constructed using $\alg_{\mathrm{track}}$. Then we will express the tracking regret as of $\alg_A$ in terms of tracking regret of $\alg_{\mathrm{track}}$. Finally we will apply bounds on tracking regret from previous work to obtain the bound of the proposition.

    The assistant's algorithm $\alg_A$ maintains $M_H$ copies of $\alg_{\mathrm{track}}$, one for each human action $a_H \in \mathcal{A}_H$. Let us denote the copy associated with action $a_H$ by $\alg_{\mathrm{track}}^{a_H}$.  Each copy of $\alg_{\mathrm{track}}$ optimizes over the space of assistant actions $\mathcal{A}_A$. The assistant's policy at round $t$: $\pi_A^{(t)}$ chooses the action selected by the copy  $\alg_{\mathrm{track}}^{a_H^{(t)}}$, where $a_H^{(t)} = \pi_H^{(t)}(\theta^{(t)})$. At the end of round $t$, the assistant updates this copy of $\alg_{\mathrm{track}}$ using the bandit feedback it receives of reward $r_t(\pi_H^{(t)}, \pi_A^{(t)})$ upon selecting action $a_A^{(t)}$.

    We will show that we can bound the assistant's tracking regret over the space $\Pi_A$ by the sum of the tracking regrets of each of the $M_H$ copies of $\alg_{\mathrm{track}}$. We can decompose the reward $r_t$ as $r_t(\overline{\pi}_H^{(t)}, \pi_A) = \sum_{a_H \in \mathcal{A}_H} r_t^{a_H}(\pi_A(a_H))$,
    where each $r_t^{a_H}: \mathcal{A}_A \to [0,1]$ is defined by $r_t^{a_H}(a_A) = r(a_H, a_A; \theta^{(t)}) \cdot \mathbf{1}\left(\overline{\pi}_H^{(t)}(\theta^{(t)}) = a_H\right)$.
     $r_t^{a_H}$ is the reward function with which the copy $\alg_{\mathrm{track}}^{a_H}$ is updated. We can decompose the tracking regret of the assistant as
    \begin{align*}
        R_T^{\text{track}}(\alg_A;\, p) &= \max_{\substack{s_1, \ldots, s_{p+1} \\ \pi_A^{(1)}, \ldots, \pi_A^{(p)} \in \Pi_A}} \sum_{i=1}^{p} \sum_{t=s_i}^{s_{i+1}-1} r_t(\overline{\pi}_H^{(t)}, \pi_A^{(i)}) - \sum_{t=1}^T r_t(\overline{\pi}_H^{(t)}, \pi_A^{(t)}) \\
        &= \max_{\substack{s_1, \ldots, s_{p+1} \\ \pi_A^{(1)}, \ldots, \pi_A^{(p)}}} \sum_{a_H \in \mathcal{A}_H} \left[\sum_{i=1}^{p} \sum_{t=s_i}^{s_{i+1}-1} r_t^{a_H}(\pi_A^{(i)}(a_H)) - \sum_{t=1}^T r_t^{a_H}(\pi_A^{(t)}(a_H))\right] \\
        &\leq \sum_{a_H \in \mathcal{A}_H} \max_{\substack{s_1, \ldots, s_{p+1} \\ \pi_A^{(1)}, \ldots, \pi_A^{(p)}}} \left[\sum_{i=1}^{p} \sum_{t=s_i}^{s_{i+1}-1} r_t^{a_H}(\pi_A^{(i)}(a_H)) - \sum_{t=1}^T r_t^{a_H}(\pi_A^{(t)}(a_H))
        \right] \\
        &\le \sum_{a_H \in \mathcal{A}_H} R_T^{\text{track}}(\alg_{\mathrm{track}}^{a_H};\, p) = M_H R_T^{\text{track}}(\alg_{\mathrm{track}}^{a_H};\, p).
    \end{align*}

    Using regret bounds for a tracking regret minimizing algorithm from previous work like Fixed-Share algorithm~\citep{herbster1998tracking,cesa2012mirror} that we restate in \Cref{sec:tracking-regret}, we get the tracking regret bound in the proposition.
\qed

\subsection{Proof of \Cref{prop:natural-regret}}\label{proof:natural-regret}
    First we describe the decentralized human and assistant yielding this assistance regret rate. The human algorithm is derived from the centralized algorithm $\alg_C$ with the external regret and switching bound provided in \Cref{prop:stable-centralized}. In particular, $\alg_C$ has expected $(1-1/e)$-external regret at most $\bO{M_A M_H \sqrt{T}}$ and, with probability $1 - \nicefrac{1}{T}$, switches at most $\bO{\sqrt{M_AM_HT} \log(T)}$ times. $\alg_H$ is derived from $\alg_C$ in the standard way described in \Cref{alg:CAH}.
    The assistant algorithm $\alg_A$ is the one shown to achieve low tracking regret in \Cref{prop:black-box-tracking-regret} and the human algorithm $\alg_H$ derived from the centralized algorithm.

    We can immediately apply \Cref{lem:stable-adaptive} to achieve the regret guarantee stated in the theorem, for $\alpha = 1 - \nicefrac{1}{e}$,
    \begin{align*}
        R_T^\alpha &\leq R_T^{\alpha, \mathrm{ext}}(\alg_C) + R_T^{\mathrm{track}}(\alg_A; S_T(\alg_C; \delta)) + \delta T && (\text{By \Cref{lem:stable-adaptive}})\\
        &= R_T^{\alpha, \mathrm{ext}}(\alg_C) + M_H\sqrt{M_A T \log(M_A T)\cdot S_T(\alg_C; \nicefrac{1}{T})} + 1 && (\text{By \Cref{prop:black-box-tracking-regret}})\\
        &= M_AM_H\sqrt{T} + M_H\sqrt{M_A T \log(M_A T)\cdot \sqrt{M_AM_HT}\log(T)} + 1 && (\text{By \Cref{prop:stable-centralized}})\\
        &\in \bO{M_A M_H \sqrt{T} + M_H^{5/4} M_A^{3/4}\, T^{3/4} \sqrt{\log(M_A T)\log(T)}} \subseteq \bO{M_A M_H \sqrt{T} + M_H^{5/4} M_A^{3/4}\log(M_A T)\, T^{3/4}}. &&\qedhere
    \end{align*}
\qed

\subsection{Proof of \Cref{thm:main-algorithm}}\label{proof:main-algorithm}
We first describe the human and assistant algorithms $\alg_H$ and $\alg_A$, and then bound their assistance regret.

\paragraph{Construction of $\alg_H$ and $\alg_A$.}
As in the proofs of \Cref{lem:stable-adaptive,prop:black-box-tracking-regret}, the human algorithm $\alg_H$ runs the stable centralized algorithm $\alg_C$ from \Cref{prop:stable-centralized}. On most rounds it plays the human component of $\alg_C$'s output. However, whenever $\alg_C$ switches to a new policy pair, $\alg_H$ temporarily deviates in order to communicate the new assistant policy to the assistant. This removes the need for the assistant to relearn the policy through exploration, enabling us to get a better regret rate than the $O(T^{3/4})$ rate in \Cref{prop:black-box-tracking-regret}.

An assistant policy is a mapping $\pi_A : \mathcal{A}_H \to \mathcal{A}_A$. There are $M_A^{M_H}$ such policies, so any policy can be encoded using
$\log_{M_H}(M_A^{M_H}) = M_H \log_{M_H}(M_A)$
human actions. We assume the agents share a common mapping sequences of human actions to assistant policies.

\paragraph{Signaling a switch.}
To begin communication, both agents must detect that a switch has occurred. The human knows this (since it runs $\alg_C$), but the assistant does not. We design the following signaling-string protocol in which the human uses a special sequence of actions to indicate that a switch has occurred.

Before the first round of the game, but after nature chooses a sequence of preferences $(\theta^{(1)},\dots,\theta^{(T)})$, the human and assistant sample a shared random string
$\sigma \in \mathcal{A}_H^\ell$, for $\ell = \left\lceil 2\log_{M_H} T \right\rceil$, where each element of $\sigma$ that is drawn
uniformly at random from $\mathcal{A}_H$. The human chooses actions according to string $\sigma$ to indicate to the assistant the start of a communication phase.

Let us call the rounds outside of the communication and signaling string $\sigma$ rounds of normal play. These are rounds where the human's action is derived from the centralized algorithm. That is,
$a_H^{(t)} = \overline{\pi}_H^{(t)}(\theta^{(t)})$.

Whenever $\alg_C$ switches to a new policy pair, the human instead plays 1) the signaling string $\sigma$ for $\ell$ rounds, and
2) the encoded assistant policy for $M_H \log_{M_H}(M_A)$ rounds. The assistant monitors a sliding window of the last $\ell$ human actions and, upon detecting $\sigma$, interprets the subsequent $M_H\log_{M_H}(M_A)$ actions as an encoding of the new assistant policy.

\paragraph{No false positives in assistant's detection of communication phase.}
We now show it is unlikely that the assistant incorrectly perceives the start of a communication phase. That is $\sigma$ is unlikely to appear accidentally during normal play.

Because the adversary is oblivious, the sequence $(\theta^{(1)},\dots,\theta^{(T)})$ is fixed before $\sigma$ is sampled. Moreover, $\alg_C$ is full-information: its updates depend only on past preferences and its own randomness, and not on the realized actions. Fix any realization $\rho$ of $\alg_C$'s internal randomness. Given $(\theta^{(1)},\dots,\theta^{(T)})$ and $\rho$, the sequence of policies output by $\alg_C$ is deterministic. Hence the sequence of human actions during normal play $\mathbf{a_H}^{\mathrm{norm}}(\rho)$
is deterministic and does not depend on $\sigma$.

Since $\sigma$ is drawn uniformly from $\mathcal{A}_H^\ell$ and independently of $\mathbf{a}^{\mathrm{norm}}(\rho)$, the probability that any fixed window of $\ell$ consecutive normal-play actions equals $\sigma$ is $M_H^{-\ell}$. Taking a union bound over at most $T$ such windows,
\[
\Pr \left[\sigma \text{ appears during normal play}\right]
\le T \cdot M_H^{-\ell}
\le T \cdot T^{-2}
= \frac{1}{T}.
\]
This bound holds for every $\rho$, and hence unconditionally. Therefore, with probability at least $1 - 1/T$, the assistant detects every communication phase correctly and incurs no false positives.

After receiving a communicated policy, the assistant plays it exactly until the next switch, achieving perfect synchronization.

\paragraph{Assistance regret analysis.}
To get the final bound on assistance regret, we will use an approach similar to the decomposition from \Cref{lem:stable-adaptive}. Conditioning on the event that $\sigma$ does not appear during normal play (which holds with probability at least $1-1/T$), the policies played by the decentralized algorithms match the policies selected by $\alg_C$ in all rounds but 1) the rounds where signaling string $\sigma$ is played to indicate switch, and 2) the rounds used to communicate the switched assistant policy. For a single switch, these rounds occur $\ell = \log_{M_H} T$ times and $M_H \log_{M_H} M_A$ times respectively.
Since rewards lie in $[0,1]$, each such round incurs at most unit regret. If $\alg_C$ makes $S_T$ switches, then these rounds result in a cost of coordination of
\[
R_{\mathrm{coord}}
\le (\ell + M_H\log_{M_H}(M_A))\, S_T
= O\!\bigl(M_H\log_{M_H}(M_A) + \log_{M_H} T\bigr)\, S_T.
\]

On the $1/T$-probability event of a false positive, we bound regret by $T$, contributing at most $1$ in expectation.

By \Cref{prop:stable-centralized}, the centralized algorithm satisfies
\[
R_T^{\mathrm{ext}}(\alg_C)\in O(M_A M_H \sqrt{T}),
\qquad
S_T\in O\!\bigl(\sqrt{M_A M_H T}\,\log(T/\delta)\bigr)
\]
with probability at least $1-\delta$. Setting $\delta = 1/T$ and combining terms,
\begin{align*}
R_T
&\le R_T^{\mathrm{ext}}(\alg_C) + R_{\mathrm{coord}} + \delta T \\
&\le O(M_A M_H \sqrt{T})
  + O\!\bigl(M_H\log_{M_H}(M_A)+\log_{M_H}T\bigr)
    \cdot O\!\bigl(\sqrt{M_A M_H T}\,\log T\bigr)
  + O(1) \\
&= O\!\bigl(M_A M_H \sqrt{T} + (M_H \log M_A + \log T)\, \sqrt{M_A M_H T}\, \log T\bigr) \subseteq \bO{M_A M_H \sqrt{T} + M_H\sqrt{M_H M_A T}\, \log(M_A T)\log T}. \qedhere
\end{align*}

\section{Results from Previous Work}\label{sec:results_previous_work}
In this section, we will state the results from previous work that we use for our results. We will also prove any extensions to these results that we need for our results.

\subsection{Online Submodular Maximization}

A key tool that we use is the reduction of online submodular maximization to online convex optimization constructed in \citet{salem2024online}. This reduction relies on the existence of concave relaxations and rounding schemes that satisfy a ``sandwich property.'' The following proposition establishes that this property holds for Lipschitz weighted threshold potentials (\Cref{def:weighted-coverage}), which is the class of functions arising in assistance games.

\begin{proposition}[Sandwich Property for Lipschitz Weighted Threshold Potentials]
    \label{prop:sandwich-wtp}
    Let $\mathcal{X} \subseteq \{0,1\}^n$ and let $\mathcal{Y} = \mathrm{conv}(\mathcal{X})$ be its convex hull. Let $\mathcal{F}$ be a class of weighted threshold potential functions (\Cref{def:weighted-coverage}) over $\mathcal{X}$ that are $L$-Lipschitz. Then, there exists a randomized rounding $\Xi: \mathcal{Y} \times U \to \mathcal{X}$, where $U$ is a random variable, such that for every $f \in \mathcal{F}$, there exists an $L$-Lipschitz concave function $\tilde{f}: \mathcal{Y} \to \mathbb{R}$ satisfying:
    \begin{enumerate}
        \item $\tilde{f}(\mathbf{x}) \geq f(\mathbf{x})$ for all $\mathbf{x} \in \mathcal{X}$, and
        \item $\mathbb{E}_{\Xi}[f(\Xi(\mathbf{y}))] \geq (1 - 1/e) \cdot \tilde{f}(\mathbf{y})$ for all $\mathbf{y} \in \mathcal{Y}$.
    \end{enumerate}
\end{proposition}

This proposition is stated as Lemma 1 in \citet{salem2024online} and its proof is provided in Appendix E of~\citet{salem2024online}. The proof involves showing that negatively correlated randomized roundings lead to the property stated in the proposition for weighted threshold potential functions with concave relaxations that are Lipschitz. It then uses results from \citet{chekuri2010dependent} to show that swap rounding and randomized pipage rounding are negatively correlated for partition matroids $\mathcal{M}$. We describe the swap rounding algorithm and express it as a function over $\mathcal{X}$ and a random variable $U \sim \mathrm{Uniform}([0,1])^{n+1}$.

\begin{algorithm}[H]
\caption{Swap Rounding \citep{chekuri2010dependent,salem2024online}}
\label{alg:swap-rounding}
\begin{algorithmic}[1]
\Require $\mathbf{y} \in \mathcal{Y} = \mathrm{conv}(\mathcal{X})$, random vector $\mathbf{U} = (U_1, \ldots, U_{n+1}) \sim \mathrm{Unif}([0,1])^{n+1}$
\Ensure $\mathbf{x} \in \mathcal{X}$
\State Decompose $\mathbf{y} = \sum_{k=1}^{K} \gamma_k \mathbf{z}_k$ where $\gamma_k \in [0,1]$, $\sum_{k=1}^{K} \gamma_k = 1$, and $\mathbf{z}_k \in \mathcal{X}$ are bases of the matroid \Comment{$K \leq n+1$ by Carath\'{e}odory}
\State $\beta_1 \gets \gamma_1$
\State $\mathbf{x} \gets \mathbf{z}_1$
\For{$k = 1, \ldots, K-1$}
    \State $\beta_{k+1} \gets \beta_k + \gamma_{k+1}$
    \If{$U_k \leq \beta_k / \beta_{k+1}$}
        \State $\mathbf{x} \gets \mathbf{x}$ \Comment{Probability $\beta_k/\beta_{k+1}$: keep current}
    \Else
        \State $\mathbf{x} \gets \mathbf{z}_{k+1}$ \Comment{Probability $\gamma_{k+1}/\beta_{k+1}$: switch}
    \EndIf
\EndFor
\State \Return $\mathbf{x}$
\end{algorithmic}
\end{algorithm}

The algorithm provided by~\cite{salem2024online} using this reduction is called Randomness Augmented Online Convex Optimization (RAOCO). We use a slightly modified version called the Coupled-Randomness RAOCO (CR-RAOCO) algorithm. CR-RAOCO implements the RAOCO algorithm in ~\citet{salem2024online} but by coupling the randomness in the rounding step across the different rounds. This is done by drawing a single random variable $\mathbf{U} \sim \mathrm{Uniform}([0,1])^{n+1}$ and using it to round all the points across all the rounds. We introduce the coupling to enhance the stability of the RAOCO algorithm. We provide a description of the CR-RAOCO algorithm in the following algorithm \Cref{alg:cr-raoco}.

\begin{algorithm}[H]
\caption{Coupled-Randomness RAOCO (CR-RAOCO)}
\label{alg:cr-raoco}
\begin{algorithmic}[1]
\Require OCO policy $\mathcal{P}_{\mathcal{Y}}$, randomized rounding $\Xi: \mathcal{Y} \times U \to \mathcal{X}$, randomness distribution $\mathcal{D}_{random}$
\State $\mathbf{U} \sim \mathcal{D}_{random}$ \Comment{Sample randomness once}
\For{$t = 1, 2, \ldots, T$}
    \State $\mathbf{y}_t \gets \mathcal{P}_{\mathcal{Y},t}((\mathbf{y}_s)_{s < t}, (\tilde{f}_s)_{s < t})$
    \State $\mathbf{x}_t \gets \Xi(\mathbf{y}_t, \mathbf{U})$ \Comment{Round using shared randomness $\mathbf{U}$}
    \State Play $\mathbf{x}_t$ and receive reward $f_t(\mathbf{x}_t)$
    \State Reward function $f_t$ is revealed
    \State Construct concave extension $\tilde{f}_t$ from $f_t$ as in \Cref{prop:sandwich-wtp}
\EndFor
\end{algorithmic}
\end{algorithm}

The following proposition shows that the regret of the CR-RAOCO policy transfers from the OCO regret guarantee. The analysis is basically the same as the regret analysis of RAOCO done by ~\cite{salem2024online}. The only difference is establishing that the sandwich property shown in~\cref{prop:sandwich-wtp} holds with the coupled rounding procedure as well. The sandwich property is only a condition on the expectation of a function applied to the rounded random variable. Since the marginal distribution of the rounded random variable is unchanged after the coupling, the expectation remains the same and the property continues to hold. We state the regret analysis including this additional argument below.

\begin{proposition}[Regret Bound for RAOCO~\citep{salem2024online}]
    \label{prop:raoco-regret}
    Under the Sandwich Property (\Cref{prop:sandwich-wtp}), given an OCO policy $\mathcal{P}_{\mathcal{Y}}$ operating over $\mathcal{Y} = \mathrm{conv}(\mathcal{X})$, the RAOCO policy $\mathcal{P}_{\mathcal{X}}$ described by \Cref{alg:cr-raoco} satisfies
    \[
        \alpha\text{-}\mathrm{regret}_T(\mathcal{P}_{\mathcal{X}}) \leq \alpha \cdot \mathrm{regret}_T(\mathcal{P}_{\mathcal{Y}}).
    \]
\end{proposition}

\begin{proof}
    Consider the sequence of reward functions $\{f_1, f_2, \ldots, f_T\} \in \mathcal{F}^T$ and the associated sequence of concave relaxations $\{\tilde{f}_1, \tilde{f}_2, \ldots, \tilde{f}_T\} \in \tilde{\mathcal{F}}^T$. Then,
    \begin{equation}
        \max_{\mathbf{x} \in \mathcal{X}} \sum_{t=1}^{T} f_t(\mathbf{x}) \leq \max_{\mathbf{x} \in \mathcal{X}} \sum_{t=1}^{T} \tilde{f}_t(\mathbf{x}) \leq \max_{\mathbf{y} \in \mathcal{Y}} \sum_{t=1}^{T} \tilde{f}_t(\mathbf{y}).
        \label{eq:raoco-upper}
    \end{equation}
    The first inequality follows from the upper bound property $\tilde{f}(\mathbf{x}) \geq f(\mathbf{x})$ for all $\mathbf{x} \in \mathcal{X}$ (property 1 in \Cref{prop:sandwich-wtp}), and the second inequality holds because maximizing over a superset $\mathcal{Y} = \mathrm{conv}(\mathcal{X}) \supseteq \mathcal{X}$ can only increase the objective value attained.

    The total expected reward obtained by the RAOCO policy is given by
    \begin{equation}
        \mathbb{E}_{\Xi}\left[\sum_{t=1}^{T} f_t(\mathbf{x}_t)\right] = \sum_{t=1}^{T} \mathbb{E}_{\Xi}[f_t(\mathbf{x}_t)] = \sum_{t=1}^{T} \mathbb{E}_{\Xi}[f_t(\Xi(\mathbf{y}_t))] \geq \alpha \sum_{t=1}^{T} \tilde{f}_t(\mathbf{y}_t),
        \label{eq:raoco-lower}
    \end{equation}
    where the inequality follows from the rounding property $\mathbb{E}_{\Xi}[f(\Xi(\mathbf{y}))] \geq \alpha \cdot \tilde{f}(\mathbf{y})$ for all $\mathbf{y} \in \mathcal{Y}$ (property 2 in \Cref{prop:sandwich-wtp}).

    Combining \Cref{eq:raoco-upper,eq:raoco-lower}, we obtain
    \begin{align*}
        \alpha \max_{\mathbf{x} \in \mathcal{X}} \sum_{t=1}^{T} f_t(\mathbf{x}) - \mathbb{E}_{\Xi}\left[\sum_{t=1}^{T} f_t(\mathbf{x}_t)\right] &\leq \alpha \max_{\mathbf{y} \in \mathcal{Y}} \sum_{t=1}^{T} \tilde{f}_t(\mathbf{y}) - \alpha \sum_{t=1}^{T} \tilde{f}_t(\mathbf{y}_t).
    \end{align*}

    Taking the supremum over all sequences of reward functions in $\mathcal{F}^T$ on both sides yields the desired result:
    \[
        \alpha\text{-}\mathrm{regret}_T(\mathcal{P}_{\mathcal{X}}) \leq \alpha \cdot \mathrm{regret}_T(\mathcal{P}_{\mathcal{Y}}).
    \]
\end{proof}

The following proposition establishes that the number of switches made by CR-RAOCO is exactly the number of switches made by the underlying OCO policy. This is a key property that allows us to control the switching cost of the algorithm.

\begin{proposition}[Switching Bound for CR-RAOCO]
    \label{prop:cr-raoco-switches}
    Let $\mathcal{P}_{\mathcal{Y}}$ be an OCO policy and let $\mathcal{P}_{\mathcal{X}}$ be the CR-RAOCO policy described by \Cref{alg:cr-raoco}. Let $(\mathbf{y}_1, \ldots, \mathbf{y}_T)$ be the sequence of fractional points produced by $\mathcal{P}_{\mathcal{Y}}$ and let $(\mathbf{x}_1, \ldots, \mathbf{x}_T)$ be the sequence of integral points produced by $\mathcal{P}_{\mathcal{X}}$. Then, the number of switches satisfies
    \[
        \sum_{t=2}^{T} \mathbbm{1}[\mathbf{x}_t \neq \mathbf{x}_{t-1}] \leq \sum_{t=2}^{T} \mathbbm{1}[\mathbf{y}_t \neq \mathbf{y}_{t-1}].
    \]
\end{proposition}

\begin{proof}
    In CR-RAOCO, the randomness $\mathbf{U}$ is sampled once at the beginning and shared across all rounds. Given this fixed $\mathbf{U}$, the rounding function $\Xi(\cdot, \mathbf{U}): \mathcal{Y} \to \mathcal{X}$ is a deterministic function. Therefore, $\mathbf{x}_t = \Xi(\mathbf{y}_t, \mathbf{U})$ for all $t$.

    Since $\Xi(\cdot, \mathbf{U})$ is deterministic, if $\mathbf{y}_t = \mathbf{y}_{t-1}$, then $\mathbf{x}_t = \Xi(\mathbf{y}_t, \mathbf{U}) = \Xi(\mathbf{y}_{t-1}, \mathbf{U}) = \mathbf{x}_{t-1}$. Equivalently, $\mathbf{x}_t \neq \mathbf{x}_{t-1}$ implies $\mathbf{y}_t \neq \mathbf{y}_{t-1}$. Thus,
    \[
        \mathbbm{1}[\mathbf{x}_t \neq \mathbf{x}_{t-1}] \leq \mathbbm{1}[\mathbf{y}_t \neq \mathbf{y}_{t-1}]
    \]
    for all $t \geq 2$, and summing over $t$ yields the result.
\end{proof}

\subsection{Online Convex Optimization with Few Switches}\label{sec:lazy-oco}

The CR-RAOCO algorithm uses an OCO algorithm as a black box. To ensure CR-RAOCO has both low regret and few switches, we use the Private Continuous Online Multiplicative Weights with Euclidean Regularization (POMER) algorithm of \citet{agarwal2024lazy}, a Shrinking-Dartboard-style algorithm based on sampling from a strongly log-concave density. Unlike the lazy OCO algorithm of \citet{sherman2021lazy}, which requires the loss functions to be $\beta$-smooth, POMER requires only that the losses be convex and Lipschitz---properties that hold for the value-of-assistance concave relaxation $\tilde{V}_\theta$ used in our setting (\Cref{lem:assistance-is-Lipschitz}). Compared to the earlier Shrinking-Dartboard variant of \citet{anava2015memory}, POMER additionally provides a flexible switching budget, an improved dimensional dependence in its regret bound, and a high-probability bound on the number of switches that is included directly in the analysis of \citet{agarwal2024lazy}.

\begin{algorithm}[H]
\caption{Private Continuous Online Multiplicative Weights with Euclidean Regularization (POMER)~\citep{agarwal2024lazy}}
\label{alg:p-comw}
\begin{algorithmic}[1]
\Require Convex domain $W \subseteq \mathbb{R}^d$, temperature $\beta > 0$, regularization $\lambda > 0$, scale $\Phi > 0$
\State Sample $x_1 \sim \mu_1$ where $\mu_1(x) \propto \exp\!\bigl(-\tfrac{\lambda}{2}\|x\|^2\bigr)$
\For{$t = 1, \ldots, T$}
    \State Play $x_t$ and observe loss $g_t : W \to [0,1]$
    \State Define the strongly log-concave density $\mu_{t+1}(x) \propto \exp\!\Bigl(-\beta \sum_{\tau=1}^{t} g_\tau(x) - \tfrac{\lambda}{2}\|x\|^2\Bigr)$
    \State With probability $\min\!\bigl\{1,~ \mu_{t+1}(x_t)/(\Phi \cdot \mu_t(x_t))\bigr\}$, set $x_{t+1} \gets x_t$ \Comment{Stay}
    \State Otherwise, sample $x_{t+1} \sim \mu_{t+1}$ independently \Comment{Switch}
\EndFor
\end{algorithmic}
\end{algorithm}

POMER samples each iterate from a strongly log-concave density---continuous online multiplicative weights with an added $\ell_2$ regularization---and uses a rejection-sampling-style stay-step that ensures few switches while preserving the marginal regret of the underlying continuous multiplicative weights distribution. The strong-convexity term governed by $\lambda$ is the key innovation that buys the improved dimensional dependence relative to prior non-smooth analyses.  The following theorem from \citet{agarwal2024lazy} bounds both the regret and the number of switches as a function of an arbitrary switching budget $S$, with no smoothness assumption on the losses.

\begin{theorem}[Lazy OCO Guarantees~\citep{agarwal2024lazy}]
    \label{thm:p-comw}
    Let $W \subseteq \mathbb{R}^d$ be convex with diameter $D = \sup_{x,y \in W} \|x-y\|$, and let $g_1, \ldots, g_T : W \to [0,1]$ be convex and $G$-Lipschitz. For any switching budget $S \in [1, T]$, with parameters $\beta, \lambda, \Phi$ chosen as in Theorem 4.5 of \citet{agarwal2024lazy}, \Cref{alg:p-comw} satisfies
    \[
        R_T \;\le\; GD\sqrt{2T} \;+\; 16\,GD\log(T)\cdot \frac{\sqrt{d}\cdot T}{S} \;+\; 13\,GD,
    \]
    and the number of switches satisfies $\mathbb{E}[S_T] \le S$ together with the high-probability bound
    \[
        \Pr\!\bigl[S_T \ge 3 S\bigr] \;\le\; e^{-S}.
    \]
\end{theorem}

\paragraph{Polynomial-time implementation.} \Cref{alg:p-comw} requires sampling from the strongly log-concave density $\mu_{t+1}(x) \propto \exp\bigl(-\beta \sum_{\tau \le t} g_\tau(x) - \tfrac{\lambda}{2}\|x\|^2\bigr)$ over $W$. When $W$ admits a polynomial-time membership oracle---as is the case for the matroid polytope used in CR-RAOCO---this sampling problem can be solved approximately in time $\mathrm{poly}(d, T, 1/\varepsilon)$ via standard log-concave samplers such as hit-and-run~\citep{lovasz2007geometry}. The polynomial degree is higher than that of FTPRLL but remains $\mathrm{poly}(M_H, M_A, N, T)$, matching the runtime claim of our centralized algorithm.

\subsection{Tracking regret}\label{sec:tracking-regret}

The tracking regret framework has been studied to achieve good rewards relative to the baseline action changing $p$ times during the $T$ rounds and is related to adaptive regret~\cite{hazan2007adaptive}.

\citet{herbster1998tracking} provide an algorithm Fixed-Share that achieves $m$-segment tracking regret $\bO{\sqrt{Tm (\log N + \log T) N}}$ in the full information setting. \citet{cesa2012mirror} show that the fixed share algorithm is equivalent to online mirror descent with a particular projection. Theorem 4.1 of~\citet{shalev2012online} shows how online mirror descent with bandit information can be implemented without degradation of regret. This is stated in the following Theorem by \citet{daniely2015strongly}.

\begin{theorem}[Restatement of theorem by~\citet{daniely2015strongly}]\label{thm:fixed-share}
    Fixed-Share algorithm in the bandit feedback setting achieves expected $m$-segment tracking regret at most $\bO{\sqrt{Tm (\log N + \log T) N}}$.
\end{theorem}

\subsection{Minmax lower bound}\label{sec:t_lower_bound}

The tools used to show a minmax lower bound of $\Omega(\sqrt{T})$ regret in the standard online learning regret minimization setting can be used to derive a similar minmax lower bound of $\Omega(\sqrt{T})$ for assistance regret.

That is, we can show that for any pair of human and assistant algorithms, there is an instance of the assistance game that results in assistance regret at least $\Omega(\sqrt{T})$. The lower bound is constructed through showing limitations on distinguishing between two Bernoulli distributions with means $1/2 + \varepsilon$ and $1/2 - \varepsilon$.

At a high level, we can view the general online learning problem as a special case of the assistance game with just a single action that the human can take. The assistant's problem in this case resembles the standard external regret minimization problem.

\begin{proposition}[Assistance regret lower bound]\label{prop:t_lower_bound}
    For every human-assistant learning algorithms, there is an instance of the online assistance game for which assistant regret is $\Omega(\sqrt{T})$.
\end{proposition}
\begin{proof}
    We will show a reduction from the problem of distinguishing between two Bernoulli distributions to an assistance game.

    The $\varepsilon$-Bernoulli distinction problem is the problem of given a distribution that is one of $\mathrm{Bern}(1/2-\varepsilon)$ or $\mathrm{Bern}(1/2+\varepsilon)$, determining which distribution it is.

    Standard lower bounds for this problem state that the probability of success of any algorithm with $T$ samples in the $\varepsilon$-Bernoulli distinction problem, is at most $1 - \exp(-\varepsilon^2 T / 2)$.

    Now we will show that these lower bounds imply a lower bound for assistance regret by establishing an algorithmic reduction

    \emph{Algorithm for distinction from algorithm for the assistance game.} Consider any human-assistant algorithms for the assistance game. We can construct an algorithm for the $\varepsilon$-Bernoulli distinction problem in the following way.

    We can set up a assistance game with a single action for the human, where the preference space is $\Theta = \{0, 1\}$. The assistant's action set is $\mathcal{A} = \{0, 1\}$, and the agents receive a reward of 1 if the assistant correctly predicts the human's preference, and a reward of 0 otherwise. Preferences are drawn at the start of the game by nature according to the fixed distribution $D$ which can be either $\mathrm{Bern}(1/2-\varepsilon)$ or $\mathrm{Bern}(1/2+\varepsilon)$.

    If the assistant predicts the preference to be $1$ more than $1/2$ of the time, we conclude that the distribution is $\mathrm{Bern}(1/2+\varepsilon)$ and otherwise, we conclude the distribution is $\mathrm{Bern}(1/2-\varepsilon)$.

    Let us analyze the probability of success of this approach and relate it to the assistance regret.

    Let the assistant games induced by state generating distributions $\mathrm{Bern}(1/2-\varepsilon)$, $\mathrm{Bern}(1/2+\varepsilon)$ be $G_\varepsilon$, $G'_\varepsilon$ respectively. For a human-assistant learning algorithm, let $C_1(T)$ be a random variable denoting the number of times the assistant predicts the state to be 1 in $T$ rounds. For any assistant learning algorithm $\pi$, we can write the assistance regret under the games $G_\varepsilon, G'_\varepsilon$ as $R_T(\pi, G_\varepsilon)$ and  $R_T(\pi, G'_\varepsilon)$ respectively.

    We can bound both regrets in terms of the number of times the assistant predicts the preference 1 in the following way. In game $G_\varepsilon$ (state distribution $\mathrm{Bern}(1/2-\varepsilon)$), the optimal fixed assistant policy is to always predict $0$, giving expected reward $T(1/2+\varepsilon)$, while the algorithm's expected reward is $T(1/2+\varepsilon) - 2\varepsilon\,\mathbb{E}[C_1(T)]$ (since $a_A^{(t)} \perp \theta^{(t)}$ given the history). Markov's inequality on the non-negative variable $C_1(T)$ then gives
    \[
        R_T(\pi, G_\varepsilon) \;\geq\; 2\varepsilon\,\mathbb{E}[C_1(T)] \;\geq\; T\varepsilon \cdot \mathbb{P}_{D_\varepsilon}(C_1(T) > T/2),
    \]
    and symmetrically (with optimal policy ``predict $1$ always'' and Markov on $T - C_1(T)$),
    \[
        R_T(\pi, G'_\varepsilon) \;\geq\; T\varepsilon \cdot \mathbb{P}_{D'_\varepsilon}(C_1(T) \leq T/2).
    \]
    The events $\{C_1(T) > T/2\}$ under $D_\varepsilon$ and $\{C_1(T) \leq T/2\}$ under $D'_\varepsilon$ are exactly the events that our test guesses incorrectly, so their sum is the total failure probability of the distinction test. Adding the two regret bounds,
    \begin{align*}
        R_T(\pi, G_\varepsilon) + R_T(\pi, G'_\varepsilon) &\geq T\varepsilon \cdot \bigl[\mathbb{P}_{D_\varepsilon}(C_1(T) > T/2) + \mathbb{P}_{D'_\varepsilon}(C_1(T) \leq T/2)\bigr] \\
        &\geq T\varepsilon \cdot \text{Probability of failure in distinction problem} \\
        &\geq T\varepsilon \cdot \exp(- \varepsilon^2 T / 2),
    \end{align*}
    where the last inequality uses the standard distinction lower bound stated above (success $\leq 1 - \exp(-\varepsilon^2 T/2)$, hence failure $\geq \exp(-\varepsilon^2 T/2)$).

    Choosing $\varepsilon = \min(1/2, \sqrt{1/(4T)}) = 1/(2\sqrt{T})$ for $T \geq 1$ gives $\exp(-\varepsilon^2 T/2) = \exp(-1/8)$, so the sum of regrets is $\Omega(\sqrt{T})$, and at least one of $R_T(\pi, G_\varepsilon)$, $R_T(\pi, G'_\varepsilon)$ is $\Omega(\sqrt{T})$, which is the lower bound of the proposition.

\end{proof}

\section{Additional Related Work}\label{sec:add_related}
A particular subclass of assistance games are \emph{communication games}, which are studied in the line of work on Emergent Communication (EC) ~\citep{foerster2016learning, lazaridou2016multi, lazaridou2020emergent}. This literature studies how agents learn to communicate in Lewis signaling games when trained using standard training dynamics. Empirical work has investigated how the choice of training parameters impacts the efficiency of communication (e.g.,\citep{havrylov2017emergence,kim2021emergent, ren2019enhance, chaabouni2022emergent, rita2020lazimpa, li2019ease}) 
Convergence of training dynamics have received a more theoretical treatment in works in game theory and evolutionary biology showing ~\citep{franke2009signal, franke2009interpretation, jager2007game, jager2012game, trapa2000nash, kirby2002natural, kirby2014iterated, jacob2023regularized}, but the rates of convergence are usually not analyzed.

Training dynamics have received a more theoretical treatment in works in game theory and evolutionary biology~\citep{franke2009signal, franke2009interpretation, jager2007game, jager2012game, trapa2000nash, kirby2002natural, kirby2014iterated, jacob2023regularized}. However, these works mostly view language formation as equilibrium computation. There are many possible equilibria and there is no guarantee of convergence to the optimal equilibrium, which is the goal of our work.

\end{document}